%% file: ex_article.tex
\documentclass[review,hidelinks,onefignum,onetabnum]{siamart220329}

\input{ex_shared}

\ifpdf
\hypersetup{
  pdftitle={Quaternion tensor left ring decomposition and application for color image inpainting},
  pdfauthor={Jifei Miao, Kit Ian Kou, Hongmin Cai, and Lizhi Liu}
}
\fi

\usepackage{subfigure}
\usepackage{diagbox}
\begin{document}
\nolinenumbers

\maketitle

\begin{abstract}
In recent years, tensor networks have emerged as powerful tools for solving large-scale optimization problems. One of the most promising tensor networks is the tensor ring (TR) decomposition,  which achieves circular dimensional permutation invariance in the model through the utilization of the trace operation and equitable treatment of the latent cores. On the other hand, more recently, quaternions have gained significant attention and have been widely utilized in color image processing tasks due to their effectiveness in encoding color pixels by considering the three color channels as a unified entity. Therefore, in this paper, based on the left quaternion matrix multiplication, we propose the quaternion tensor left ring (QTLR) decomposition, which inherits the powerful and generalized representation abilities of the TR decomposition while leveraging the advantages of quaternions for color pixel representation. In addition to providing the definition of QTLR decomposition and an algorithm for learning the QTLR format, the paper further proposes a low-rank quaternion tensor completion (LRQTC) model and its algorithm for color image inpainting based on the defined QTLR decomposition. Finally, extensive experiments on color image inpainting demonstrate that the proposed LRQTC method is highly competitive.
\end{abstract}



\section{Introduction}
Tensor networks have gained prominence in recent years as powerful tools for tackling large-scale optimization problems \cite{kolda2009tensor,chen2020tensor,oseledets2011tensor,zhao2016tensor,zheng2021fully}. Among them, the tensor ring (TR) decomposition \cite{zhao2016tensor} is one of the most advanced tensor networks. The TR decomposition is a method that represents an $N$-th order tensor $\mathcal{T}\in \mathbb{R}^{I_{1}\times I_{2}\times \ldots \times I_{N}}$ by multiplying a sequence of third-order tensors $\mathcal{Z}_{n}$, $n=1,2,\ldots, N$ in a circular manner. Specifically, it can be expressed in an element-wise form given by
\begin{equation}\label{trformat}
	\mathcal{T}(i_{1},i_{2},\ldots,i_{N})={\rm{Tr}}\{\mathcal{Z}_{1}(i_{1})\mathcal{Z}_{2}(i_{2})\ldots\mathcal{Z}_{N}(i_{N})\},
\end{equation}
where $\mathcal{T}(i_{1},i_{2},\ldots,i_{N})$ denotes the $(i_{1},i_{2},\ldots,i_{N})$-th element of $\mathcal{T}$, ${\rm{Tr}}\{\cdot\}$ denotes the trace operator, $\mathcal{Z}_{n}(i_{n})\in\mathbb{R}^{r_{n}\times r_{n+1}}$  denotes the $i_{n}$-th lateral slice of the TR factor $\mathcal{Z}_{n}$, the last TR factor $\dot{\mathcal{Z}}_{N}$ is of size $r_{N}\times I_{N}\times r_{1}$, \emph{i.e.}, $r_{N+1} = r_{1}$. The TR decomposition has been widely utilized in various image processing tasks due to its powerful and generalized representation ability. In particular, the TR-based low-rank tensor completion (LRTC) methods for image inpainting have been extensively studied recently \cite{yuan2019tensor,huang2020robust,qiu2022noisy,wu2023tensor}. For example, Wang et al. presented a TR-based completion algorithm in \cite{wang2017efficient}, which involves alternately updating each TR factor. Nevertheless, the performance of the algorithm is influenced by the pre-defined TR-rank, leading to a significant increase in computational cost. In \cite{yuan2019tensor}, Yuan et al. addressed these challenges by applying matrix nuclear norm regularization to the mode-2 unfolding of each TR factor, thereby improving the stability of the performance. In \cite{yu2019tensor,huang2020provable}, the authors proposed a TR nuclear norm minimization model using a tensor circular unfolding scheme for tensor completion. Notably, this approach does not rely on a pre-defined TR-rank and demonstrates superior performance compared to previous TR decomposition-based methods. However, when dealing with color pixels comprising RGB channels, real-valued third-order tensors may not fully exploit the strong correlation among the three channels. This limitation arises from the fact that real-valued third-order tensors represent color images by simply concatenating the RGB channels together, treating both the `intra-channel relationship' (the relationship within each channel) and the `spatial relationship' (the relationship between pixels) equally \cite{miao2023quaternion1}.

On the other hand, quaternions have gained considerable attention in the field of color image processing as a more suitable tool for representing color pixels. Concretely, the quaternion-based method encodes the RGB three-channel pixel values on the three imaginary parts of a quaternion \cite{li2015quaternion}. That is
\begin{equation}
	\label{equ1}
	\dot{t}=0+t_{r}i+t_{g}j+t_{b}k,
\end{equation}
where $\dot{t}$ denotes a color pixel, $t_{r}$, $t_{g}$, and $t_{b}$ are RGB three-channel pixel values, $i$, $j$, and $k$ are the three imaginary units. While both real-valued third-order tensors and quaternion matrices can be utilized for representing color images, quaternion matrices, being a novel representation, possess more favorable characteristics and advantages in this context. Quaternions treat the three channels of color pixels as a cohesive entity \cite{miao2021color,chen2022color,jia2021structure,chen2019low}, thereby effectively preserving the intra-channel relationship. Hence, quaternion matrices, particularly their low-rank approximation models, have been extensively utilized for color image processing tasks recently. For instance, quaternion matrix rank minimization metoods \cite{jia2019robust,liu2022randomized,jia2022non,yang2022quaternion} and quaternion matrix factorization metoods \cite{chen2023quaternion,miao2021color}. These methods have achieved remarkable results in tasks such as color image inpainting and color image denoising.
While there has been a significant amount of research progress on quaternion matrices recently, the study on quaternion tensors has just begun, particularly in the context of quaternion tensor networks\footnote{The term `tensor network' is synonymous with the commonly used phrase `tensor decomposition' \cite{wang2023tensor}. Within the tensor networks family, CP, Tucker, tensor train (TT), and TR decompositions are all included.}, which remains largely unexplored. The research on quaternion tensors goes beyond a mere expansion of quaternion theory; it is primarily driven by applications. Most notably, for color videos, a third-order quaternion tensor is required for representation in the most intuitive manner \cite{he2023eigenvalues,miao2020low}. Additionally, there has been a proliferation of techniques for data reshuffling or dimensionality enhancement, such as KA \cite{bengua2017efficient}, OKA \cite{zhang2022effective}, Hankelization \cite{zheng2021tensor}, and methods that leverage non-local data similarity to rearrange image or video data \cite{li2019low}. These methods typically bring certain prior information within the original data into clearer focus, leading to substantial improvements in the final processing outcomes for specific tasks. However, these techniques generally result in an increase in data dimensions. Therefore, for color data, the exploration of higher-order quaternion arrays, namely higher-order quaternion tensors, becomes particularly essential. Currently, research on quaternion tensors is primarily focused on their singular value decomposition methods \cite{qin2022singular,miao2023quaternion}. However, the singular value decomposition of quaternion tensors often involves a high computational workload, making it unsuitable for processing large-scale data in practical applications. Thus, it is necessary to study the theory of quaternion tensor networks, which involves representing large quaternion tensors using relatively smaller quaternion tensors. This approach helps alleviate the challenges of storage and processing of large-scale data.

Consequently, in this paper, we aim to propose the quaternion tensor left ring (QTLR)\footnote{
	The term `left' originates from our use of left quaternion matrix multiplication (\emph{see} Definition \ref{landrqtmm})  to define quaternion tensor ring decomposition.} decomposition, which will inherit the powerful and generalized representation capabilities of the TR decomposition while leveraging the advantages of quaternions for color pixel representation. It is important to note that QTLR diverges from a straightforward extension of TR to quaternions, primarily due to the non-commutativity of quaternion multiplication. This inherent non-commutativity gives rise to disparities in the definitions and associated properties of QTLR in comparison to TR, which is the rationale behind the introduction of QTLR as a distinct concept from TR. Furthermore, as an important application of the defined QTLR, we propose a low-rank quaternion tensor completion (LRQTC) method based on QTLR decomposition to address the inpainting task in color images. This method can mitigate the limitations of quaternion matrix-based approaches, which are not suitable for higher-dimensional quaternion data, and tensor-based methods, which may be unable to distinguish between the intra-channel relationship and spatial relationship of color pixels. Therefore, the proposed QTLR-based LRQTC method is expected to make further advancements in the inpainting of color images compared to existing methods.

We outline the primary contributions of this paper as follows:
\begin{itemize}
	\item We define the QTLR decomposition for quaternion tensors and prove its cyclic permutation property. It is worth noting that when quaternion tensors degenerate into real tensors, the definition of QTLR decomposition and the cyclic permutation property will degenerate into their corresponding real counterparts as presented in \cite{zhao2016tensor}.
	 Furthermore, inspired by the TR-SVD algorithm introduced in \cite{zhao2016tensor}, we also present the QTLR-QSVD algorithm for learning the QTLR format.
	\item We generalize the tensor circular unfolding scheme from \cite{yu2019tensor,huang2020provable} to quaternion tensors and define the quaternion tensor circular unfolding. For the circular unfolding quaternion matrices that satisfy a certain condition, we prove the relationship between their rank and the defined QTLR-rank. Based on this, we propose a LRQTC model along with its corresponding computational method, which can be considered as an important application of QTLR decomposition.
	\item We extend a tensor augmentation technique called OKA \cite{zhang2022effective} to quaternion matrices, enabling the transformation of quaternion matrices into higher-order quaternion tensors. Subsequently, we apply the proposed LRQTC method to color image inpainting tasks. Experimental results validate the competitiveness of it.
\end{itemize}

The reminder of this paper is organized as follows. In Section \ref{sec:preliminary}, we present certain notations and foundational concepts pertaining to quaternion algebra. This encompasses quaternion matrices and quaternion tensors. In Section \ref{sec:main_1}, we provide the definition of QTLR along with its associated properties. Additionally, within this section, we present a learning algorithm for the QTLR format, termed QTLR-QSVD. In Section \ref{sec:main_2}, we propose an LRQTC model along with its corresponding algorithm. Section \ref{sec:main_3_3} outlines the concrete process of color image inpainting and presents the experimental results. The conclusion is ultimately provided in Section \ref{sec:main_5}.

\section{Preliminary}
\label{sec:preliminary}
Within this section, we introduce specific notations and fundamental principles related to the realm of quaternion algebra, covering quaternion matrices and quaternion tensors.

\subsection{Notations}
In this paper, $\mathbb{R}$, $\mathbb{C}$, and $\mathbb{H}$ respectively denote the real space, complex space, and quaternion space. A scalar, a vector, a matrix, and a tensor are written as $a$, $\mathbf{a}$, $\mathbf{A}$, and $\mathcal{A}$ respectively. $\dot{a}$,  $\dot{\mathbf{a}}$, $\dot{\mathbf{A}}$, and $\dot{\mathcal{A}}$ respectively represent a quaternion scalar, a quaternion vector, a quaternion matrix, and a quaternion tensor. The $(i_{1},i_{2},\ldots,i_{N})$-th element of $\dot{\mathcal{A}}\in\mathbb{H}^{I_{1}\times I_{2} \times\ldots \times I_{N}}$ is denoted as $\dot{\mathcal{A}}(i_{1},i_{2},\ldots,i_{N})$. $(\cdot)^{\ast}$, $(\cdot)^{T}$, and $(\cdot)^{H}$ denote the conjugate, transpose, and conjugate transpose, respectively. ${\rm{rank}}(\cdot)$ and ${\rm{Tr}}\{\cdot\}$ respectively denote the rank and trace operators. $\mathfrak{R}(\cdot)$ denotes the real part of quaternion (scalar, vector, matrix, and tensor). ${\rm{diag}}(\cdot)$, ${\rm{reshape}}(\cdot)$, and ${\rm{ permute}}(\cdot)$ are command operations in \emph{MATLAB} that respectively represent the generation of a diagonal matrix, reshaping of arrays, and rearrangement of array dimensions. In addition, $\|\cdot\|_{F}$, $\|\cdot\|_{w,\ast}$, and $\langle\cdot, \cdot\rangle$ are respectively the Frobenius norm, the weighted nuclear norm \cite{DBLP:journals/ijon/YuZY19}, and the inner product operation.

\subsection{Introduction to quaternions}
Quaternion was introduced by Hamilton \cite{hamilton1866elements}. A quaternion $\dot{q}\in\mathbb{H}$   has a Cartesian form given by:
\begin{equation}
	\label{equ2}
	\dot{q}=q_{0}+q_{1}i+q_{2}j+q_{3}k,
\end{equation}
where $q_{l}\in\mathbb{R}\: (l=0,1,2,3)$ are called its components, $i$, $j$, and $k$ are
the three imaginary units related through the famous relations:
\begin{align}
	\left\{
	\begin{array}{lc}
		i^{2}=j^{2}=k^{2}=ijk=-1,\\
		ij=-ji=k,
		jk=-kj=i, 
		ki=-ik=j.
	\end{array}
	\right.
\end{align}
Quaternions have similar rules for addition, subtraction, multiplication, and division as complex numbers, as well as similar definitions for conjugation and modulus. However, the difference lies in the non-commutativity property of quaternion multiplication. That is, in general $\dot{p}\dot{q}\neq \dot{q}\dot{p}$. 

A multidimensional array or an $N$-th order tensor is named a quaternion tensor if its elements are quaternions (quaternion matrices can be regarded as second-order quaternion tensors), \emph{i.e.}, $\dot{\mathcal{T}}=(\dot{t}_{i_{1}i_{2}\ldots i_{N}})\in\mathbb{H}^{I_{1}\times I_{2} \times\ldots \times I_{N}}
=\mathcal{T}_{0}+\mathcal{T}_{1}i+\mathcal{T}_{2}j+\mathcal{T}_{3}k$, where $\mathcal{T}_{l}\in\mathbb{R}^{I_{1}\times I_{2} \times\ldots \times I_{N}}\: (l=0,1,2,3)$, $\dot{\mathcal{T}}$ is pure if $\mathcal{T}_{0}$ is a zero tensor \cite{miao2020low}. The definition of the inner product between two quaternion tensors, $\dot{\mathcal{X}}\in\mathbb{H}^{I_{1}\times I_{2} \times\ldots \times I_{N}}$ and $\dot{\mathcal{Y}}\in\mathbb{H}^{I_{1}\times I_{2} \times\ldots \times I_{N}}$, is given by: $\langle\dot{\mathcal{X}}, \dot{\mathcal{Y}}\rangle=\sum_{i_{1}=1}^{I_{1}}\sum_{i_{2}=1}^{I_{2}}\ldots\sum_{i_{N}=1}^{I_{N}}\dot{x}_{i_{1}i_{2}\ldots i_{N}}^{\ast}\dot{y}_{i_{1}i_{2}\ldots i_{N}}$. The Frobenius norm of quaternion tensor $\dot{\mathcal{X}}$ is $\|\dot{\mathcal{X}}\|_{F}=\sqrt{\langle\dot{\mathcal{X}}, \dot{\mathcal{X}}\rangle}$.

The most common approach in studying higher-order tensors is to unfold them into matrices.
Thus, we extend three unfolding methods for real tensors in \cite{zhao2016tensor} to quaternion tensors.
\begin{definition}
	(\textbf{Multi-Index Operation} \cite{qiu2022noisy}) The multi-index operation is defined as follows:
	\begin{equation*}
		\overline{i_{1}i_{2}\ldots i_{N}}=i_{1}+(i_{2}-1)I_{1}+(i_{3}-1)I_{1}I_{2}+\ldots+(i_{N}-1)\prod_{n=1}^{N-1}I_{n},
	\end{equation*}
where $i_{n}\in [I_{n}]$.
\end{definition}

\begin{definition}
	(\textbf{k-Unfolding}) Let $\dot{\mathcal{T}}\in\mathbb{H}^{I_{1}\times I_{2} \times\ldots \times I_{N}}$ be an $N$-th order quaternion tensor, the k-unfolding of $\dot{\mathcal{T}}$ is a quaternion matrix,  denoted by $\dot{\mathbf{T}}_{\langle k\rangle}$ of size $\prod_{n=1}^{k}I_{n}\times\prod_{n=k+1}^{N}I_{n}$, whose elements are defined by
	\begin{equation*}
	\dot{\mathbf{T}}_{\langle k\rangle}(m,n)=\dot{\mathcal{T}}(i_{1},i_{2},\ldots,i_{N}),	
	\end{equation*}
where $m=\overline{i_{1}i_{2}\ldots i_{k}}$, $n=\overline{i_{k+1}i_{k+2}\ldots i_{N}}$.
\end{definition}

\begin{definition}
	(\textbf{Mode-k Unfolding}) Let $\dot{\mathcal{T}}\in\mathbb{H}^{I_{1}\times I_{2} \times\ldots \times I_{N}}$ be an $N$-th order quaternion tensor, the mode-k unfolding of $\dot{\mathcal{T}}$ is a quaternion matrix,  denoted by $\dot{\mathbf{T}}_{[k]}$ of size $I_{k}\times\prod_{n\neq k}I_{n}$, whose elements are defined by
	\begin{equation*}
		\dot{\mathbf{T}}_{[k]}(i_{k},t)=\dot{\mathcal{T}}(i_{1},i_{2},\ldots,i_{N}),	
	\end{equation*}
	where $t=\overline{i_{k+1}\ldots i_{N}i_{1}\ldots i_{k-1}}$.
\end{definition}

\begin{definition}
	(\textbf{Classical Mode-k Unfolding}) Let $\dot{\mathcal{T}}\in\mathbb{H}^{I_{1}\times I_{2} \times\ldots \times I_{N}}$ be an $N$-th order quaternion tensor, the classical mode-k unfolding of $\dot{\mathcal{T}}$ is a quaternion matrix,  denoted by $\dot{\mathbf{T}}_{(k)}$ of size $I_{k}\times\prod_{n\neq k}I_{n}$, whose elements are defined by
	\begin{equation*}
		\dot{\mathbf{T}}_{(k)}(i_{k},t)=\dot{\mathcal{T}}(i_{1},i_{2},\ldots,i_{N}),	
	\end{equation*}
	where $t=\overline{i_{1}\ldots i_{k-1}i_{k+1}\ldots i_{N}}$.
\end{definition}

In order to provide the definition of QTLR decomposition, we first introduce the left and right quaternion matrix multiplications. 
\begin{definition}\label{landrqtmm}
	(\textbf{Left and Right Quaternion Matrix Multiplications} \cite{schulz2011widely}) Given two quaternion matrices $\dot{\mathbf{A}}\in \mathbb{H}^{M\times N}$ and $\dot{\mathbf{B}}\in \mathbb{H}^{N\times P}$, the left and right multiplications are respectively defined as
	\begin{equation}\label{qmp}
		(\dot{\mathbf{A}}\cdot_{L}\dot{\mathbf{B}})_{mp}=\sum_{n=1}^{N}\dot{a}_{mn}\dot{b}_{np} \quad  \text{and} \quad (\dot{\mathbf{A}}\cdot_{R}\dot{\mathbf{B}})_{mp}=\sum_{n=1}^{N}\dot{b}_{np}\dot{a}_{mn}.
	\end{equation}
\end{definition}	
Note that due to the non-commutativity of quaternion multiplication, generally $\dot{\mathbf{A}}\cdot_{L}\dot{\mathbf{B}}\neq\dot{\mathbf{A}}\cdot_{R}\dot{\mathbf{B}}$. For simplicity,  we also define $\dot{a}\cdot_{L}\dot{b}=\dot{a}\dot{b}$ and  $\dot{a}\cdot_{R}\dot{b}=\dot{b}\dot{a}$ for quaternion scalars $\dot{a}$ and $\dot{b}$. Additionally, if we do not specify whether it is left multiplication or right multiplication, it is assumed to be left multiplication, \emph{i.e., $\dot{\mathbf{A}}\dot{\mathbf{B}}=\dot{\mathbf{A}}\cdot_{L}\dot{\mathbf{B}}$}. The defined left and right quaternion matrix multiplications have the following associativity property \cite{schulz2013using}:
\begin{equation}
	(\dot{\mathbf{A}}\cdot_{L}\dot{\mathbf{B}})\cdot_{L}\dot{\mathbf{C}}=\dot{\mathbf{A}}\cdot_{L}(\dot{\mathbf{B}}\cdot_{L}\dot{\mathbf{C}}) \quad \text{and} \quad (\dot{\mathbf{A}}\cdot_{R}\dot{\mathbf{B}})\cdot_{R}\dot{\mathbf{C}}=\dot{\mathbf{A}}\cdot_{R}(\dot{\mathbf{B}}\cdot_{R}\dot{\mathbf{C}}). 
\end{equation}
However, in general
\begin{equation}\label{hunhebudeng}
	(\dot{\mathbf{A}}\cdot_{L}\dot{\mathbf{B}})\cdot_{R}\dot{\mathbf{C}}\neq\dot{\mathbf{A}}\cdot_{L}(\dot{\mathbf{B}}\cdot_{R}\dot{\mathbf{C}}) \quad \text{and} \quad (\dot{\mathbf{A}}\cdot_{R}\dot{\mathbf{B}})\cdot_{L}\dot{\mathbf{C}}\neq\dot{\mathbf{A}}\cdot_{R}(\dot{\mathbf{B}}\cdot_{L}\dot{\mathbf{C}}). 
\end{equation}
In addition, the property that ${\rm{rank}}(\dot{\mathbf{A}}\cdot_{L}\dot{\mathbf{B}})\leq\min({\rm{rank}}(\dot{\mathbf{A}}),{\rm{rank}}(\dot{\mathbf{B}}))$ has been proven in \cite{chen2020low}. In the following theorem, we demonstrate the same property for the right multiplication of two quaternion matrices.

\begin{theorem}
	For any two quaternion matrices $\dot{\mathbf{A}}\in \mathbb{H}^{M\times N}$ and $\dot{\mathbf{B}}\in \mathbb{H}^{N\times P}$, we have 
	\begin{equation}\label{ranlr1}
		{\rm{rank}}(\dot{\mathbf{A}}\cdot_{R}\dot{\mathbf{B}})\leq\min({\rm{rank}}(\dot{\mathbf{A}}),{\rm{rank}}(\dot{\mathbf{B}})).
	\end{equation}
\end{theorem}
\begin{proof}
	Denote the right row null spaces of $\dot{\mathbf{B}}$ and $\dot{\mathbf{A}}\cdot_{R}\dot{\mathbf{B}}$ as $\mathcal{RRN}(\dot{\mathbf{B}})=\{\dot{\mathbf{x}}\in\mathbb{H}^{P}:\dot{\mathbf{B}}\cdot_{R}\dot{\mathbf{x}}=\mathbf{0}\}$ and $\mathcal{RRN}(\dot{\mathbf{A}}\cdot_{R}\dot{\mathbf{B}})=\{\dot{\mathbf{x}}\in\mathbb{H}^{P}:(\dot{\mathbf{A}}\cdot_{R}\dot{\mathbf{B}})\cdot_{R}\dot{\mathbf{x}}=\mathbf{0}\}$ \cite{schulz2013using}. One can easily find that
	$\mathcal{RRN}(\dot{\mathbf{B}})\subseteq\mathcal{RRN}(\dot{\mathbf{A}}\cdot_{R}\dot{\mathbf{B}})$. Thus, $\dim\mathcal{RRN}(\dot{\mathbf{B}})\leq\dim\mathcal{RRN}(\dot{\mathbf{A}}\cdot_{R}\dot{\mathbf{B}})$ and ${\rm{rank}}(\dot{\mathbf{A}}\cdot_{R}\dot{\mathbf{B}})\leq{\rm{rank}}(\dot{\mathbf{B}})$.
	Similarly, Denote the right column null spaces of $\dot{\mathbf{A}}$ and $\dot{\mathbf{A}}\cdot_{R}\dot{\mathbf{B}}$ as $\mathcal{RCN}(\dot{\mathbf{A}})=\{\dot{\mathbf{x}}\in\mathbb{H}^{M}:\dot{\mathbf{x}}^{T}\cdot_{R}\dot{\mathbf{A}}=\mathbf{0}^{T}\}$ and $\mathcal{RCN}(\dot{\mathbf{A}}\cdot_{R}\dot{\mathbf{B}})=\{\dot{\mathbf{x}}\in\mathbb{H}^{M}:\dot{\mathbf{x}}^{T}\cdot_{R}(\dot{\mathbf{A}}\cdot_{R}\dot{\mathbf{B}})=\mathbf{0}^{T}\}$ \cite{schulz2013using}. One can also find that
	$\mathcal{RCN}(\dot{\mathbf{A}})\subseteq\mathcal{RCN}(\dot{\mathbf{A}}\cdot_{R}\dot{\mathbf{B}})$. Thus, $\dim\mathcal{RCN}(\dot{\mathbf{A}})\leq\dim\mathcal{RCN}(\dot{\mathbf{A}}\cdot_{R}\dot{\mathbf{B}})$ and ${\rm{rank}}(\dot{\mathbf{A}}\cdot_{R}\dot{\mathbf{B}})\leq{\rm{rank}}(\dot{\mathbf{A}})$. In all, ${\rm{rank}}(\dot{\mathbf{A}}\cdot_{R}\dot{\mathbf{B}})\leq\min({\rm{rank}}(\dot{\mathbf{A}}),{\rm{rank}}(\dot{\mathbf{B}}))$.
\end{proof}

\section{Quaternion tensor left ring decomposition}
\label{sec:main_1}
In this section, we first define the QTLR decomposition. Following the definition, we present an important property of the QTLR decomposition, and finally propose an algorithm for learning the QTLR format.

\subsection{The definition of QTLR decomposition}
\begin{definition}
	(\textbf{QTLR Decomposition}) Let $\dot{\mathcal{T}}\in \mathbb{H}^{I_{1}\times I_{2}\times \ldots \times I_{N}}$ be an $N$-th order quaternion tensor with $I_{n}$-dimension along the $n$-th mode, then QTLR representation is to decompose it
	into a sequence of third-order quaternion tensors $\dot{\mathcal{Z}}_{n}\in \mathbb{H}^{r_{n}\times I_{n}\times r_{n+1}}$ (which can be also called the $n$-th core of $\dot{\mathcal{T}}$),  $n=1,2,\ldots, N$, which can be represented using an element-wise formulation as\footnote{One can also use the right quaternion matrix multiplication to define the QTRR decomposition, which is $	\dot{\mathcal{T}}(i_{1},i_{2},\ldots,i_{N})={\rm{Tr}}\{\dot{\mathcal{Z}}_{1}(i_{1})\cdot_{R}\dot{\mathcal{Z}}_{2}(i_{2})\cdot_{R}\ldots\cdot_{R}\dot{\mathcal{Z}}_{N}(i_{N})\}$. Due to the analogous analysis process and application effects between QTRR and QTLR, we will exclusively consider QTLR throughout this paper.}
	\begin{equation}\label{qtr1}
		\dot{\mathcal{T}}(i_{1},i_{2},\ldots,i_{N})={\rm{Tr}}\{\dot{\mathcal{Z}}_{1}(i_{1})\cdot_{L}\dot{\mathcal{Z}}_{2}(i_{2})\cdot_{L}\ldots\cdot_{L}\dot{\mathcal{Z}}_{N}(i_{N})\},
	\end{equation}
where $\dot{\mathcal{T}}(i_{1},i_{2},\ldots,i_{N})$ denotes the $(i_{1},i_{2},\ldots,i_{N})$-th element of $\dot{\mathcal{T}}$, $\dot{\mathcal{Z}}_{n}(i_{n})=\dot{\mathcal{Z}}_{n}(:,i_{n},:)\in\mathbb{H}^{r_{n}\times r_{n+1}}$  denotes the $i_{n}$-th lateral slice quaternion matrix of the third-order quaternion tensor $\dot{\mathcal{Z}}_{n}$, the last third-order quaternion tensor $\dot{\mathcal{Z}}_{N}$ is of size $r_{N}\times I_{N}\times r_{1}$, i.e., $r_{N+1} = r_{1}$, which ensures the product of these quaternion matrices is a square quaternion matrix. In addition, the vector $\mathbf{r}=[r_{1},r_{2},\ldots,r_{N}]$ is defined as the QTLR-rank of the quaternion tensor $\dot{\mathcal{T}}$.
\end{definition}

Note that formula (\ref{qtr1}) can also be expressed in index form as follows:
\begin{equation}\label{qtr2}
	\begin{split}
	\dot{\mathcal{T}}(i_{1},i_{2},\ldots,i_{N})=&\sum_{\alpha_{1}=1}^{r_{1}}\cdots\sum_{\alpha_{N}=1}^{r_{N}}\dot{\mathcal{Z}}_{1}(\alpha_{1},i_{1},\alpha_{2})\cdot_{L}\\
	&{Z}_{2}(\alpha_{2},i_{2},\alpha_{3})\cdot_{L}\ldots\cdot_{L}{Z}_{N}(\alpha_{N},i_{N},\alpha_{N+1}),
    \end{split}
\end{equation}
where $\alpha_{N+1}=\alpha_{1}$. Thus, one can easily find that quaternion tensor train (QTT) decomposition \cite{miao2022quaternion} is a special case of the defined QTLR decomposition when $r_{1}=1$.

For an efficient representation of QTLR decomposition, we introduce two quaternion tensor multiplications for third-order quaternion tensors, namely the quaternion tensor left connection multiplication and the quaternion tensor right connection multiplication.
\begin{definition} (\textbf{Quaternion Tensor Left and Right Connection Multiplications})
Let $\dot{\mathcal{Z}}_{n}\in\mathbb{H}^{r_{n}\times I_{n}\times r_{n+1}}$, $n=1,2,\ldots,N$, be $N$ third-order quaternion tensors, the quaternion tensor left and right connection multiplications between $\dot{\mathcal{Z}}_{n}$ and $\dot{\mathcal{Z}}_{n+1}$ are respectively defined as
\begin{equation}\label{Lcp}
\dot{\mathcal{Z}}_{n}\cdot_{L}\dot{\mathcal{Z}}_{n+1}\in\mathbb{H}^{r_{n}\times I_{n}I_{n+1}\times r_{n+2}}	={\rm{reshape}}(\dot{\mathbf{Z}}_{n}^{L}\cdot_{L}\dot{\mathbf{Z}}_{n+1}^{R},[r_{n},I_{n}I_{n+1},r_{n+2}])
\end{equation}
and
\begin{equation}\label{Rcp}
	\dot{\mathcal{Z}}_{n}\cdot_{R}\dot{\mathcal{Z}}_{n+1}\in\mathbb{H}^{r_{n}\times I_{n}I_{n+1}\times r_{n+2}}	={\rm{reshape}}(\dot{\mathbf{Z}}_{n}^{L}\cdot_{R}\dot{\mathbf{Z}}_{n+1}^{R},[r_{n},I_{n}I_{n+1},r_{n+2}]),
\end{equation}
where $\dot{\mathbf{Z}}_{n}^{L}\in \mathbb{H}^{r_{n}I_{n}\times I_{n+1}}=(\dot{\mathbf{Z}}_{n})_{\langle2\rangle}$ and $\dot{\mathbf{Z}}_{n+1}^{R}\in \mathbb{H}^{r_{n+1}\times I_{n+1}I_{n+2}}=(\dot{\mathbf{Z}}_{n+1})_{\langle1\rangle}$.
\end{definition}
Then, following the definition (\ref{Lcp}), the QTLR decomposition (\ref{qtr1}) can be represented as
\begin{equation}\label{key}
\dot{\mathcal{T}}=f(\dot{\mathcal{Z}})=f(\dot{\mathcal{Z}}_{1}\cdot_{L}\dot{\mathcal{Z}}_{2}\cdot_{L}\ldots\cdot_{L}\dot{\mathcal{Z}}_{N}),	
\end{equation}
where function $f$ is a trace operation on $\dot{\mathcal{Z}}(:,k,:)$, $k=1,2,\ldots,\prod_{i=1}^{N}I_{i}$, followed by a reshaping operation from vector of the length $\prod_{i=1}^{N}I_{i}$ to quaternion tensor of the size $I_{1}\times I_{2}\times \ldots \times I_{N}$.
\begin{definition}
	(\textbf{Quaternion Tensor Permutation}) For any $N$-th order quaternion tensor $\dot{\mathcal{T}}\in \mathbb{H}^{I_{1}\times I_{2}\times \ldots \times I_{N}}$, the $n$-th quaternion tensor permutation is defined as $\dot{\mathcal{T}}^{P_{n}}\in \mathbb{H}^{I_{n}\times\ldots\times I_{N}\times I_{1}\times \ldots \times I_{n-1}}$:
	\begin{equation}\label{qtp}
		\dot{\mathcal{T}}^{P_{n}}(i_{n},\ldots,i_{N},i_{1},\ldots,i_{n-1})=\dot{\mathcal{T}}(i_{1},i_{2},\ldots,i_{N}).
	\end{equation}
\end{definition}

In the following theorem, based on the definition of quaternion tensor permutation and QTLR decomposition, we present the cyclic permutation property of QTLR decomposition.
\begin{theorem}
(\textbf{Cyclic Permutation Property of QTLR Decomposition}) The quaternion tensor permutation of $\dot{\mathcal{T}}$ is equivalent to its cores circularly shifting, as follows:
\begin{equation}\label{cpp}
\dot{\mathcal{T}}^{P_{n}}=f\big((\dot{\mathcal{Z}}_{n}\cdot_{L}\ldots\cdot_{L}\dot{\mathcal{Z}}_{N})\cdot_{R}(\dot{\mathcal{Z}}_{1}\cdot_{L}\ldots\cdot_{L}\dot{\mathcal{Z}}_{n-1})\big),
\end{equation}
with elements
\begin{equation}\label{cppe}
	\begin{split}
&\dot{\mathcal{T}}^{P_{n}}(i_{n},\ldots,i_{N},i_{1},\ldots,i_{n-1})\\
&={\rm{Tr}}\{\big(\dot{\mathcal{Z}}_{n}(i_{n})\cdot_{L}\ldots\cdot_{L}\dot{\mathcal{Z}}_{N}(i_{N})\big)\cdot_{R}\big(\dot{\mathcal{Z}}_{1}(i_{1})\cdot_{L}\ldots\cdot_{L}\dot{\mathcal{Z}}_{n-1}(i_{n-1})\big)\}.	
	\end{split}
\end{equation}
\end{theorem}
\begin{proof}
\begin{equation}\label{cppf}
	\begin{split}
		&\dot{\mathcal{T}}^{P_{n}}(i_{n},\ldots,i_{N},i_{1},\ldots,i_{n-1})\\
		&=\sum_{\alpha_{1}=1}^{r_{1}}\cdots\sum_{\alpha_{N}=1}^{r_{N}}\dot{\mathcal{Z}}_{1}(\alpha_{1},i_{1},\alpha_{2})\cdot_{L}{Z}_{2}(\alpha_{2},i_{2},\alpha_{3})\cdot_{L}\ldots\cdot_{L}{Z}_{N}(\alpha_{N},i_{N},\alpha_{N+1})\\
		&=\sum_{\alpha_{1}=1}^{r_{1}}\cdots\sum_{\alpha_{N}=1}^{r_{N}}\big(\dot{\mathcal{Z}}_{1}(\alpha_{1},i_{1},\alpha_{2})\cdot_{L}\ldots\cdot_{L}{Z}_{n-1}(\alpha_{n-1},i_{n-1},\alpha_{n})\big)\cdot_{L}\\
		&\qquad\big({Z}_{n}(\alpha_{n},i_{n},\alpha_{n+1})\cdot_{L}\ldots\cdot_{L}{Z}_{N}(\alpha_{N},i_{N},\alpha_{N+1}) \big)\\
		&=\sum_{\alpha_{1}=1}^{r_{1}}\cdots\sum_{\alpha_{N}=1}^{r_{N}}\big({Z}_{n}(\alpha_{n},i_{n},\alpha_{n+1})\cdot_{L}\ldots\cdot_{L}{Z}_{N}(\alpha_{N},i_{N},\alpha_{N+1})\big)\cdot_{R}\\
		&\qquad\big(\dot{\mathcal{Z}}_{1}(\alpha_{1},i_{1},\alpha_{2})\cdot_{L}\ldots\cdot_{L}{Z}_{n-1}(\alpha_{n-1},i_{n-1},\alpha_{n})\big)\\
		&={\rm{Tr}}\{\big(\dot{\mathcal{Z}}_{n}(i_{n})\cdot_{L}\ldots\cdot_{L}\dot{\mathcal{Z}}_{N}(i_{N})\big)\cdot_{R}\big(\dot{\mathcal{Z}}_{1}(i_{1})\cdot_{L}\ldots\cdot_{L}\dot{\mathcal{Z}}_{n-1}(i_{n-1})\big)\},
	\end{split}
\end{equation}
where the first equality holds due to (\ref{qtp}) and (\ref{qtr2}), the second and third equalities hold directly as a result of the definitions of left and right multiplications between quaternion scalars.
\end{proof}

Note that,  when quaternion tensors degenerate into real tensors, the cyclic permutation property will degenerate into its real counterpart as presented in \cite{zhao2016tensor}. Although, due to the non-commutativity of quaternion multiplication, there are significant differences in the cyclic permutation property between our defined QTLR decomposition and the TR decomposition in \cite{zhao2016tensor}, they exhibit a similar form, which is why we refer to the decomposition of (\ref{qtr1}) as quaternion tensor left `ring'.

In the following, we develop an algorithm to learn the
QTLR format.
\subsection{QTLR-QSVD algorithm}
Considering that exact quaternion tensor decompositions often demand extensive computational resources and storage, our focus shifts towards low-rank quaternion tensor approximation within the QTLR format.
Inspired by the TR-SVD algorithm for TR decomposition in \cite{zhao2016tensor}, we propose QTLR-QSVD algorithm for learning the QTLR format in this section. Before deriving the QTLR-QSVD algorithm, we first present a required definition and a theorem.
\begin{definition}
	(\textbf{Quaternion Subchain Tensors}) Four quaternion subchain tensors are defined and denoted by
	\begin{equation}\label{qsts}
		\begin{split}
	\dot{\mathcal{Z}}^{<k}\in \mathbb{H}^{r_{1}\times \prod_{n=1}^{k-1}I_{n}\times r_{k}}&=\dot{\mathcal{Z}}_{1}\cdot_{L}\dot{\mathcal{Z}}_{2}\cdot_{L}\ldots\cdot_{L}\dot{\mathcal{Z}}_{k-1},\\
		\dot{\mathcal{Z}}^{\leq k}\in \mathbb{H}^{r_{1}\times \prod_{n=1}^{k}I_{n}\times r_{k+1}}&=\dot{\mathcal{Z}}_{1}\cdot_{L}\dot{\mathcal{Z}}_{2}\cdot_{L}\ldots\cdot_{L}\dot{\mathcal{Z}}_{k},\\	
	\dot{\mathcal{Z}}^{>k}\in \mathbb{H}^{r_{k+1}\times \prod_{n=k+1}^{N}I_{n}\times r_{1}}&=\dot{\mathcal{Z}}_{k+1}\cdot_{L}\dot{\mathcal{Z}}_{k+2}\cdot_{L}\ldots\cdot_{L}\dot{\mathcal{Z}}_{N},\\
	\dot{\mathcal{Z}}^{\geq k}\in \mathbb{H}^{r_{k}\times \prod_{n=k}^{N}I_{n}\times r_{1}}&=\dot{\mathcal{Z}}_{k}\cdot_{L}\dot{\mathcal{Z}}_{k+1}\cdot_{L}\ldots\cdot_{L}\dot{\mathcal{Z}}_{N}.
       \end{split}
	\end{equation}
\end{definition}
Note that the lateral slice matrices of $\dot{\mathcal{Z}}^{<k}$ are $\dot{\mathcal{Z}}^{<k}(:,t,:)=\prod_{n=1}^{k-1}\dot{\mathcal{Z}}_{n}(i_{n})$, where $t=	\overline{i_{1}i_{2}\ldots i_{k-1}}$. Similar
results can be obtained for $\dot{\mathcal{Z}}^{\leq k}$, $\dot{\mathcal{Z}}^{>k}$, and $\dot{\mathcal{Z}}^{\geq k}$.
\begin{theorem}\label{zyxz}
	 Assume $\dot{\mathcal{T}}$ can be represented by a QTLR decomposition. Then,
	 \begin{equation*}
	 \dot{\mathbf{T}}_{\langle k\rangle}=\dot{\mathbf{Z}}^{\leq k}_{(2)}\cdot_{L}(\dot{\mathbf{Z}}^{>k}_{[2]})^{T}.	
	 \end{equation*}
\end{theorem}
\begin{proof}
	Based on the definition of k-unfolding of quaternion tensor $\dot{\mathcal{T}}$ , we can express the QTLR decomposition in the following form:
	\begin{equation}\label{kuf}
		\begin{split}
			\dot{\mathbf{T}}_{\langle k\rangle}(t_1,t_2)&={\rm{Tr}}\{\dot{\mathcal{Z}}_{1}(i_{1})\cdot_{L}\dot{\mathcal{Z}}_{2}(i_{2})\cdot_{L}\ldots\cdot_{L}\dot{\mathcal{Z}}_{N}(i_{N})\}\\
			&={\rm{Tr}}\left\{\prod_{n=1}^{k}\dot{\mathcal{Z}}_{n}(i_{n})\prod_{n=k+1}^{N}\dot{\mathcal{Z}}_{n}(i_{n})\right\}\\
			&={\rm{Tr}}\left\{\dot{\mathcal{Z}}^{\leq k}(:,t_1,:) \dot{\mathcal{Z}}^{>k}(:,t_2,:)\right\}\\
			&={\rm{reshape}}(\dot{\mathcal{Z}}^{\leq k}(:,t_1,:),[1,r_{1}r_{k+1}])\cdot_{L}\\
			&\qquad{\rm{reshape}}((\dot{\mathcal{Z}}^{>k}(:,t_2,:))^{T},[r_{1}r_{k+1},1])\\
			&=\sum_{p=1}^{r_{1}r_{k+1}}\dot{\mathbf{Z}}^{\leq k}_{(2)}(t_1,p)\cdot_{L}((\dot{\mathbf{Z}}^{>k}_{[2]})^{T})(p,t_2),
		\end{split}
	\end{equation}
where $t_1=\overline{i_{1}i_{2}\ldots i_{k}}$ and $t_2=\overline{i_{k+1}i_{k+2}\ldots i_{N}}$. Thus, we have $\dot{\mathbf{T}}_{\langle k\rangle}=\dot{\mathbf{Z}}^{\leq k}_{(2)}\cdot_{L}(\dot{\mathbf{Z}}^{>k}_{[2]})^{T}$.
\end{proof}

Now, we present an algorithm that utilizes $N$ sequential quaternion singular value decompositions (QSVDs) \cite{zhang1997quaternions} for computing the QTLR decomposition. From theorem \ref{zyxz}, we have $\dot{\mathbf{T}}_{\langle 1\rangle}=\dot{\mathbf{Z}}^{\leq 1}_{(2)}\cdot_{L}(\dot{\mathbf{Z}}^{>1}_{[2]})^{T}$, then we truncate the QSVD\footnote{The $\delta$-truncated QSVD of a quaternion matrix $\dot{\mathbf{T}}$ means that we use a truncation threshold $\delta$ to truncate the singular values of $\dot{\mathbf{T}}$, retaining only those that are greater than or equal to $\delta^{2}$. The symbol ${\rm{rank}}_{\delta}(\dot{\mathbf{T}})$ represents the number of singular values in $\dot{\mathbf{T}}$ that are greater than or equal to $\delta^{2}$.} of $\dot{\mathbf{T}}_{\langle 1\rangle}$ to obtain its low-rank approximation, \emph{i.e.}, such that
\begin{equation}\label{alg1}
	\dot{\mathbf{T}}_{\langle 1\rangle}=\dot{\mathbf{U}}_{1}\Sigma_{1}\dot{\mathbf{V}}_{1}^{H}+\dot{\mathbf{\varepsilon}}_{1}.
\end{equation}
Let $\dot{\mathbf{Z}}^{\leq 1}_{(2)}=\dot{\mathbf{U}}_{1}$ and $(\dot{\mathbf{Z}}^{>1}_{[2]})^{T}=\Sigma_{1}\dot{\mathbf{V}}_{1}^{H}$, then the first core $\dot{\mathcal{Z}}_{1}$ and quaternion subchain tensor $\dot{\mathcal{Z}}^{>1}$ can be obtained by the proper reshaping and permutation of $\dot{\mathbf{U}}_{1}$ and $\Sigma_{1}\dot{\mathbf{V}}_{1}^{H}$, respectively. Afterwards, let $\dot{\mathbf{Z}}^{>1}={\rm{reshape}}(\dot{\mathcal{Z}}^{>1},[r_{2}I_{2},\prod_{n=3}^{N}I_{n}r_{1}])$, then truncate the QSVD of $\dot{\mathbf{Z}}^{>1}$ to obtain its low-rank approximation, \emph{i.e.}, such that
\begin{equation}\label{alg2}
	\dot{\mathbf{Z}}^{>1}=\dot{\mathbf{U}}_{2}\Sigma_{2}\dot{\mathbf{V}}_{2}^{H}+\dot{\mathbf{\varepsilon}}_{2}.
\end{equation}
Then, the second core  $\dot{\mathcal{Z}}_{2}$ and quaternion subchain tensor $\dot{\mathcal{Z}}^{>2}$ can be obtained by the proper reshaping of $\dot{\mathbf{U}}_{2}$ and $\Sigma_{2}\dot{\mathbf{V}}_{2}^{H}$, respectively. This procedure can be carried out in a sequential manner to acquire all $N$ cores $\dot{\mathcal{Z}}_{n}, n=1,2,\ldots,N$. Similar to the TR-SVD algorithm \cite{zhao2016tensor}, for QTLR-QSVD algorithm, we set the truncation threshold as  
\begin{equation}\label{deltan}
	\delta_{n}=	\left\{
	\begin{array}{lc}
		\sqrt{2}\epsilon_{p}\|\dot{\mathcal{T}}\|_{F}/\sqrt{N}, \qquad &n=1,  \\
		\epsilon_{p}\|\dot{\mathcal{T}}\|_{F}/\sqrt{N},  \qquad &n>1,
	\end{array}
	\right.
\end{equation}
where $\epsilon_{p}$ is a given prescribed relative error.
The detailed procedure of the QTLR-QSVD algorithm is listed in Algorithm \ref{alg:qtrqsvd}.
\begin{algorithm}
	\caption{QTLR-QSVD}
	\label{alg:qtrqsvd}
	\begin{algorithmic}
		\REQUIRE An $N$-th order quaternion tensor $\dot{\mathcal{T}}\in \mathbb{H}^{I_{1}\times I_{2}\times \ldots \times I_{N}}$ and the
		prescribed relative error $\epsilon_{p}$.
		\STATE Step1: Compute truncation threshold $\delta_{n}$ for $n=1$ and $n>1$ via (\ref{deltan}).
		\STATE Step2: Choose the first mode as the start point and obtain $1$-unfolding quaternion matrix  $\dot{\mathbf{T}}_{\langle 1\rangle}$.
		\STATE Step3: Low-rank approximation by applying $\delta_{1}$-truncated QSVD: $ \dot{\mathbf{T}}_{\langle 1\rangle}=\dot{\mathbf{U}}_{1}\Sigma_{1}\dot{\mathbf{V}}_{1}^{H}+\dot{\mathbf{\varepsilon}}_{1}$.
		\STATE Step4: Split ranks $r_{1}$ and $r_{2}$ by:
		$\mathop{{\rm{min}}}\limits_{r_{1},r_{2}}\ |r_{1}-r_{2}|,\ \text{\emph{s.t.}}\ r_{1}r_{2}={\rm{rank}}_{\delta_{1}}(\dot{\mathbf{T}}_{\langle 1\rangle})$.
		\STATE Step5: Obtain $\dot{\mathcal{Z}}_{1}$ via $\dot{\mathcal{Z}}_{1}={\rm{permute}}({\rm{reshape}}(\dot{\mathbf{U}}_{1},[I_{1},r_{1},r_{2}]),[2,1,3])$.
		\STATE Step6: Obtain $\dot{\mathcal{Z}}^{>1}$ via $\dot{\mathcal{Z}}^{>1}={\rm{permute}}({\rm{reshape}}(\Sigma_{1}\dot{\mathbf{V}}_{1}^{H},[r_{1},r_{2},\prod_{n=2}^{N}I_{n}]),[2,3,1])$.
		\STATE Step7: Perform the following iterative procedure:
		\FOR {$n=2$ to $N-1$}
		\STATE $\dot{\mathbf{Z}}^{>n-1}={\rm{reshape}}(\dot{\mathcal{Z}}^{>n-1},[r_{n}I_{n},\prod_{p=n+1}^{N}I_{p}r_{1}])$.
		\STATE Compute $\delta_{n}$-truncated QSVD: $\dot{\mathbf{Z}}^{>n-1}=\dot{\mathbf{U}}_{n}\Sigma_{n}\dot{\mathbf{V}}_{n}^{H}+\dot{\mathbf{\varepsilon}}_{n}$.
		\STATE  $r_{n+1}={\rm{rank}}_{\delta_{n}}(\dot{\mathbf{Z}}^{>n-1})$.
		\STATE $\dot{\mathcal{Z}}_{n}={\rm{reshape}}(\dot{\mathbf{U}}_{n},[r_{n},I_{n},r_{n+1}])$.
    	\STATE $\dot{\mathcal{Z}}^{>n}={\rm{reshape}}(\Sigma_{n}\dot{\mathbf{V}}_{n}^{H},[r_{n+1},\prod_{p=n+1}^{N}I_{p},r_{1}])$.
		\ENDFOR
		\ENSURE  $N$ cores $\dot{\mathcal{Z}}_{n}, n=1,2,\ldots,N$ of QTLR decomposition.
	\end{algorithmic}
\end{algorithm}

The QTLR-QSVD algorithm possesses inherent computational efficiency as a result of its non-recursive nature, enabling it to achieve a high degree of approximation for any given quaternion tensor. We compared the reconstruction (approximation) performance of QTLR-QSVD and TR-SVD on color images in Figure \ref{QTR-QSVD_TR-SVD}. 
\begin{figure}[htbp]
	\centering
	\subfigure[]{
		\includegraphics[width=12.5cm,height=5cm]{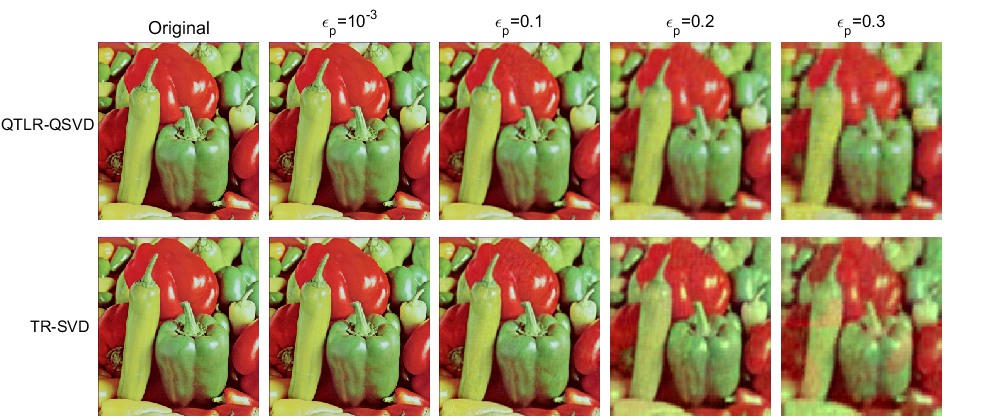}
	}
	\subfigure[]{
		\includegraphics[width=12.8cm,height=5cm]{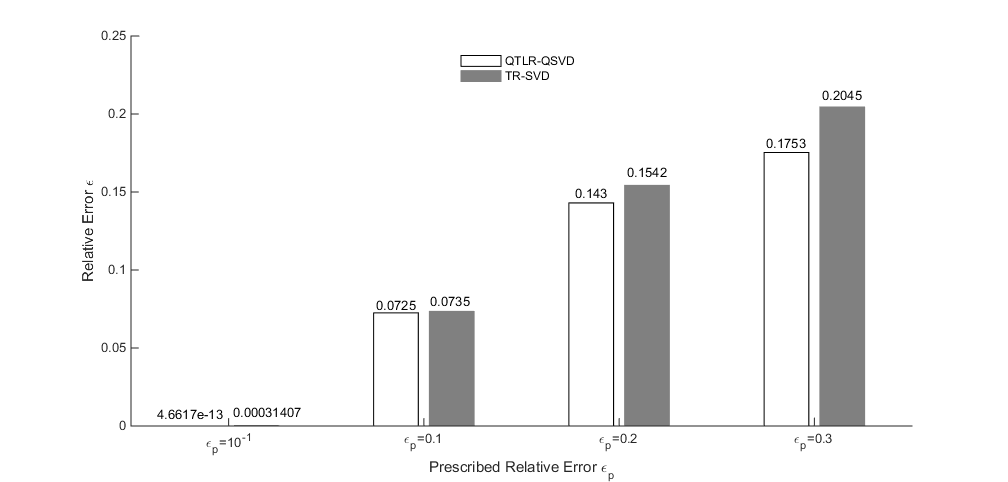}
	}
	\caption{The reconstruction of one color image `peppers' by using QTLR-QSVD and TR-SVD. The color image is tensorized to $9$th-order quaternion tensor $4\times4\times\ldots\times4$ ($10$th-order tensor $4\times4\times\ldots\times4\times3$ for TR-SVD) by the OKA procedure (see Section \ref{sec:main_3}). (a) The reconstruction results are visually displayed  for different prescribed relative errors. (b) The reconstruction relative errors are shown for different prescribed relative errors.
	}
	\label{QTR-QSVD_TR-SVD}
\end{figure}
Based on the comparison results, we can conclude that the incorporation of quaternions enables QTLR-QSVD to achieve better performance in the reconstruction of color images compared to TR-SVD. Similar results can be obtained for other color images as well.

Note that the non-commutativity of quaternion multiplication prevents the learning algorithm for the QTLR model from achieving the same level of richness as the algorithm for the TR model, which is essentially a degenerate version of the QTLR model in the real number domain. For instance, the TR-ALS series algorithms used for learning the TR model in \cite{zhao2016tensor} cannot be applied to the learning of the QTLR model\footnote{In fact, one can verify that based on our definition of QTLR and its satisfied cyclic permutation property, the inability of the TR-ALS algorithms \cite{zhao2016tensor} to be applied to learning the QTLR format primarily stems from property (\ref{hunhebudeng}).}.

As mentioned in the introduction section, the introduction of the QTLR model primarily aims to combine the structural advantages of the TR model with the benefits of quaternion representation for color pixels. It is anticipated that methods based on the QTLR decomposition will yield improved results in tasks related to color image processing. In this paper, we will use the example of color image inpainting based on an LRQTC model as an application case to validate this assertion.

\section{Low-rank quaternion tensor completion}\label{sec:main_2}
In this section, we will propose an LRQTC model and its corresponding optimization algorithm based on the previously defined QTLR decomposition and QTLR-rank.

\subsection{LRQTC model based on QTLR weighted nuclear norm minimization}
In order to introduce the QTLR weighted nuclear norm formulation, we first define the circular unfolding of a quaternion tensor and then theoretically establish its connection to the QTLR-rank. This method does not rely on a pre-specified QTLR-rank, thus transforming the problem of approximating higher-order quaternion tensors with a low QTLR-rank into a low-rank approximation problem of quaternion matrices.
\begin{definition}
(\textbf{Quaternion Tensor Circular Unfolding}) Let $\dot{\mathcal{T}}\in \mathbb{H}^{I_{1}\times I_{2}\times \ldots \times I_{N}}$ be an $N$-th order quaternion tensor, its circular unfolding is a quaternion matrix, denoted by
$\dot{\mathbf{T}}_{\{k,l\}}$, which first permutes $\dot{\mathcal{T}}$ with order $[k,\ldots,N,1\ldots,k-1]$ and then performs matricization along the first $l$ modes, i.e., $l$-unfolding. The indices of $\dot{\mathbf{T}}_{\{k,l\}}(p,q)$ are formulated as
\begin{equation}\label{indices}
	\dot{\mathbf{T}}_{\{k,l\}}(p,q)=\dot{\mathcal{T}}(i_{1},i_{2},\ldots,i_{N}),
\end{equation}
where $p=1+\sum_{s=k}^{k+l-1}(i_{s}-1)\prod_{t=k}^{s-1}I_{t}$ and $ q=1+\sum_{s=k+l}^{k-1}(i_{s}-1)\prod_{t=k+l}^{s-1}I_{t}$. Additionally, we use ${\rm{fold}}_{\{k,l\}} (\dot{\mathbf{T}}_{\{k,l\}})$ to denote the inverse process of quaternion tensor circular unfolding. 
\end{definition}
Note that when $l=1$, the quaternion tensor circular unfolding is reduced to the quaternion tensor mode-$k$ unfolding of $\dot{\mathcal{T}}$, \emph{i.e.}, $\dot{\mathbf{T}}_{\{k,1\}}=\dot{\mathbf{T}}_{[k]}$.

\begin{theorem}\label{theor1}
Assume that $\dot{\mathcal{T}}\in \mathbb{H}^{I_{1}\times I_{2}\times \ldots \times I_{N}}$ is an $N$-th order quaternion tensor with QTLR-rank $\mathbf{r}=[r_{1},r_{2},\ldots,r_{N}]$, and then when $l=N-k+1$, for each $\dot{\mathbf{T}}_{\{k,l\}}$, we have
\begin{equation}\label{rankthe}
	{\rm{rank}}(\dot{\mathbf{T}}_{\{k,l\}})\leq r_{k}r_{k+l}.
\end{equation}
\end{theorem}
\begin{proof}
	From (\ref{qtr1}) and (\ref{indices}), $\dot{\mathbf{T}}_{\{k,l\}}(p,q)$ can be represented in the index form, that is 
\begin{equation}\label{rankproof1}
	\dot{\mathbf{T}}_{\{k,l\}}(p,q)=\dot{\mathcal{T}}(i_{1},i_{2},\ldots,i_{N})
	={\rm{Tr}}\{\dot{\mathcal{Z}}_{1}(i_{1})\cdot_{L}\dot{\mathcal{Z}}_{2}(i_{2})\cdot_{L}\ldots\cdot_{L}\dot{\mathcal{Z}}_{N}(i_{N})\}.
\end{equation}
When $l=N-k+1$, (\ref{rankproof1}) can be rewritten as
\begin{equation}\label{proofc1}
	\begin{split}
	\dot{\mathbf{T}}_{\{k,l\}}(p,q)&= {\rm{Tr}}\{\big(\dot{\mathcal{Z}}_{k}(i_{k})\cdot_{L}\ldots\cdot_{L}\dot{\mathcal{Z}}_{N}(i_{N})\big)\cdot_{R}\big(\dot{\mathcal{Z}}_{1}(i_{1})\cdot_{L}\ldots\cdot_{L}\dot{\mathcal{Z}}_{k-1}(i_{k-1})\big)\}\\	
	&={\rm{Tr}}\{\big(\dot{\mathcal{Z}}_{k}(i_{k})\cdot_{L}\ldots\cdot_{L}\dot{\mathcal{Z}}_{k+l-1}(i_{k+l-1})\big)\cdot_{R}\big(\dot{\mathcal{Z}}_{1}(i_{1})\cdot_{L}\ldots\cdot_{L}\dot{\mathcal{Z}}_{k-1}(i_{k-1})\big)\}\\	
	&=	{\rm{Tr}}\{\dot{\mathcal{W}}(:,p,:)\cdot_{R}\dot{\mathcal{H}}(:,q,:)\}\\
	&=\sum_{\alpha_{1}=1}^{r_{k}}\sum_{\alpha_{2}=1}^{r_{l+k}}\dot{\mathcal{W}}(\alpha_{1},p,\alpha_{2})\cdot_{R}\dot{\mathcal{H}}(\alpha_{2},q,\alpha_{1})\\
	&=\sum_{\beta=1}^{r_{k}r_{l+k}}\dot{\mathbf{W}}_{(2)}(p,\beta)\cdot_{R}\dot{\mathbf{H}}_{[2]}^{T}(\beta,q)\\
	&=\sum_{\beta=1}^{r_{k}r_{l+k}}\dot{\mathbf{H}}_{[2]}^{T}(\beta,q)\dot{\mathbf{W}}_{(2)}(p,\beta),
	\end{split}
\end{equation}
where $\dot{\mathcal{W}}\in\mathbb{H}^{r_{k}\times\prod_{n=k}^{N}I_{n}\times r_{l+k}}=\dot{\mathcal{Z}}_{k}\cdot_{L}\ldots\cdot_{L}\dot{\mathcal{Z}}_{k+l-1}$, $\dot{\mathcal{H}}\in\mathbb{H}^{r_{1}\times\prod_{n=1}^{k-1}I_{n}\times r_{k}}=\dot{\mathcal{Z}}_{1}\cdot_{L}\ldots\cdot_{L}\dot{\mathcal{Z}}_{k-1}$. According to (\ref{proofc1}), we can get that $\dot{\mathbf{T}}_{\{k,l\}}=\dot{\mathbf{W}}_{(2)}\cdot_{R}\dot{\mathbf{H}}_{[2]}^{T}$, which combining (\ref{ranlr1}) means ${\rm{rank}}(\dot{\mathbf{T}}_{\{k,l\}})\leq \min({\rm{rank}}(\dot{\mathbf{W}}_{(2)},{\rm{rank}}(\dot{\mathbf{H}}_{[2]}^{T}))\leq\min(r_{k}r_{k+l},\prod_{n=k}^{N}I_{n},\prod_{n=1}^{k-1}I_{n})$. Note that $r_{k+l}=r_{N+1}=r_{1}$.
\end{proof}

\subsubsection{The proposed model}
From Theorem \ref{theor1}, we can observe that for an arbitrary $N$-th order quaternion tensor with QTLR-rank $\mathbf{r}=[r_{1},r_{2},\ldots,r_{N}]$, the rank of each circular unfolding quaternion matrix $\dot{\mathbf{T}}_{\{k,l\}}$ with $l=N-k+1$ is bounded by $r_{k}r_{k+l}$. Hence, the problem of quaternion tensor QTLR-rank minimization can be equivalently transformed into a sequence of quaternion matrix rank minimization subproblems, \emph{i.e.}, to minimize QTLR-rank, a natural
option is to consider the sum of rank of circular unfolding quaternion matrices:
\begin{equation}\label{qtt_rank}
	\mathop{{\rm{min}}}\limits_{\dot{\mathcal{T}}}\ \sum_{k=2}^{N}\alpha_{k}{\rm{rank}}(\dot{\mathbf{T}}_{\{k,l\}}),
\end{equation}
where $\alpha_{k}$ for $k=2,3,\ldots,N$ are positive parameters satisfying $\sum_{k=2}^{N}\alpha_{k}=1$. Note that $k$ starts from $2$, because when $k=1$, there is no permutation for $\dot{\mathcal{T}}$. Nevertheless, the general computational complexity of problem (\ref{qtt_rank}) makes it intractable. To address the solvability of (\ref{qtt_rank}), a convex surrogate, the sum of weighted nuclear norm, has been adopted. The definition of this surrogate is provided as follows.

\begin{definition}
(\textbf{QTLR Weighted Nuclear Norm}) Assume the quaternion tensor $\dot{\mathcal{T}}$ with
QTLR decomposition, its QTLR weighted nuclear norm is defined as
\begin{equation}\label{qtrnn}
	\sum_{k=2}^{N}\alpha_{k}\|\dot{\mathbf{T}}_{\{k,l\}}\|_{w,\ast},
\end{equation}
where $l=N-k+1$.
\end{definition}

Folowing Theorem \ref{theor1}, we constrain $l = N-k+1$ in our defined QTLR weighted nuclear norm, and we perform permutations on $\dot{\mathcal{T}}$  for $k = 2$ to $k = N$. Therefore, these $N-1$ differently sized and permuted quaternion matrices allow for a more comprehensive capture of the low-rank structure of $\dot{\mathcal{T}}$ and the global information of the quaternion data.

Based on the defined QTLR weighted nuclear norm (\ref{qtrnn}), we propose the following LRQTC model:
\begin{equation}
	\label{equmodel1}
	\begin{split}
		&\mathop{{\rm{min}}}\limits_{\dot{\mathcal{T}}}\ \sum_{k=2}^{N}\alpha_{k}\|\dot{\mathbf{T}}_{\{k,l\}}\|_{w,\ast}\\ 
		&\ \text{s.t.}\ P_{\Omega}(\dot{\mathcal{T}})=P_{\Omega}(\dot{\mathcal{X}}),
	\end{split}
\end{equation}
where $\dot{\mathcal{T}}\in \mathbb{H}^{I_{1}\times I_{2}\times \ldots \times I_{N}}$  is a completed output $N$-th order quaternion tensor,  $\dot{\mathcal{X}}\in \mathbb{H}^{I_{1}\times I_{2}\times \ldots \times I_{N}}$ is the observed $N$-th order quaternion tensor, and $P_{\Omega}(\cdot)$ is the projection
operator on $\Omega$ which is the index of observed elements. Specifically,
\begin{equation*}
	\mathcal{P}_{\Omega}(\dot{\mathcal{T}})=\left\{
	\begin{array}{lc}
		\dot{\mathcal{T}}(i_{1},i_{2},\ldots,i_{N}),\qquad &(i_{1},i_{2},\ldots,i_{N})\in \Omega, \\
		0,  &\text{otherwise}.
	\end{array}
	\right.
\end{equation*}

\subsubsection{Numercial scheme to solve the LRQTC model}
To enable the solution of (\ref{equmodel1}), we use the variable-splitting technique and introduce auxiliary quaternion tensors $\{\dot{\mathcal{M}}^{(k)}\}_{k=2}^{N}\in \mathbb{H}^{I_{1}\times I_{2}\times \ldots \times I_{N}}$ in (\ref{equmodel1}). Consequently, (\ref{equmodel1}) is finally transformed into the following solvable model:
\begin{equation}
	\label{equmodel2}
	\begin{split}
		&\mathop{{\rm{min}}}\limits_{\dot{\mathcal{T}},\{\dot{\mathcal{M}}^{(k)}\}}\ \sum_{k=2}^{N}\alpha_{k}\|\dot{\mathbf{M}}^{(k)}_{\{k,l\}}\|_{w,\ast}\\ 
		&\ \text{s.t.}\ \dot{\mathcal{T}}=\dot{\mathcal{M}}^{(k)},\  k=2,3,\ldots, N,\\ &\qquad \quad  P_{\Omega}(\dot{\mathcal{T}})=P_{\Omega}(\dot{\mathcal{X}}).
	\end{split}
\end{equation}
Based on the ADMM framework in the quaternion domain \cite{DBLP:journals/tsp/MiaoK20}, the augmented Lagrangian function of (\ref{equmodel2}) is defined as
\begin{equation}\label{equqlf}
	\begin{split}
		\mathcal{L}_{\mu}(\dot{\mathcal{X}},\{\dot{\mathcal{M}}^{(k)}\}_{k=2}^{N},\{\dot{\mathcal{Y}}^{(k)}\}_{k=2}^{N})
		=&\sum_{k=2}^{N}\alpha_{k}\|\dot{\mathbf{M}}^{(k)}_{\{k,l\}}\|_{w,\ast}
		+\mathfrak{R}(\langle\dot{\mathcal{Y}}^{(k)},\dot{\mathcal{T}}-\dot{\mathcal{M}}^{(k)}\rangle)\\
		&+\frac{\mu_{k}}{2}\|\dot{\mathcal{T}}-\dot{\mathcal{M}}^{(k)}\|_{F}^{2} \\
		\  \text{s.t.}\ P_{\Omega}(\dot{\mathcal{T}})=&P_{\Omega}(\dot{\mathcal{X}}),
	\end{split}	
\end{equation}
where $\dot{\mathcal{Y}}^{(k)}\in \mathbb{H}^{I_{1}\times I_{2}\times \ldots \times I_{N}}$ for $k=2,3,\ldots,N$ are Lagrange Multipliers, $\mu_{k}>0$ for $k=2,3,\ldots,N$ are penalty parameters. Then, we use an iterative scheme to solve the problem (\ref{equqlf}).

\textbf{Update $\dot{\mathcal{M}}^{(k)}$}: To optimize $\dot{\mathcal{M}}^{(k)}$ is equivalent to solve the subproblem:
\begin{equation}\label{equupm1}
	\begin{split}
		\dot{\mathcal{M}}^{(k)}&=\mathop{{\rm{arg\, min}}}\limits_{\dot{\mathcal{M}}^{(k)}}\ \alpha_{k}\|\dot{\mathbf{M}}^{(k)}_{\{k,l\}}\|_{w,\ast}
		+\mathfrak{R}(\langle\dot{\mathcal{Y}}^{(k)},\dot{\mathcal{T}}-\dot{\mathcal{M}}^{(k)}\rangle)+\frac{\mu_{k}}{2}\|\dot{\mathcal{T}}-\dot{\mathcal{M}}^{(k)}\|_{F}^{2} \\
		&=\mathop{{\rm{arg\, min}}}\limits_{\dot{\mathcal{M}}^{(k)}}\ \frac{\alpha_{k}}{\mu_{k}}\|\dot{\mathbf{M}}^{(k)}_{\{k,l\}}\|_{w,\ast}
		+\frac{1}{2}\|\dot{\mathcal{M}}^{(k)}-(\dot{\mathcal{T}}+\frac{\dot{\mathcal{Y}}^{(k)}}{\mu_{k}})\|_{F}^{2}.
	\end{split}
\end{equation}
Denote $\dot{\mathbf{\Gamma}}$= $\dot{\mathcal{T}}+\frac{\dot{\mathcal{Y}}^{(k)}}{\mu_{k}}$ and let $\dot{\mathbf{\Gamma}}=\dot{\mathbf{U}}\mathbf{\Sigma}\dot{\mathbf{V}}^{H}$ be the QSVD of $\dot{\mathbf{\Gamma}}$, where 
\begin{equation*}
	\mathbf{\Sigma}=\left[\begin{array}{cc}
		{\rm{diag}}\big(\sigma_{1}(\dot{\mathbf{\Gamma}}), \ldots, \sigma_{s}(\dot{\mathbf{\Gamma}})\big)\\ 
		\mathbf{0}
	\end{array} \right],
\end{equation*}
and $\sigma_{n}(\dot{\mathbf{\Gamma}})$ is the $n$-th singular value of $\dot{\mathbf{\Gamma}}$, $s$ denotes the number of nonzero singular values of $\dot{\mathbf{\Gamma}}$. From \cite{DBLP:journals/ijon/YuZY19}, the problem (\ref{equupm1}) has the following closed-form solution:
\begin{equation}\label{equupm2}
	\dot{\mathcal{M}}^{(k)}={\rm{fold}}_{\{k,l\}}
	(\dot{\mathbf{U}}\hat{\mathbf{\Sigma}}\dot{\mathbf{V}}^{H}),
\end{equation}
where 
\begin{equation*}
	\hat{\mathbf{\Sigma}}=\left[\begin{array}{cc}
		{\rm{diag}}\big(\sigma_{1}(\dot{\mathbf{M}}^{(k)}_{\{k,l\}}),\ldots, \sigma_{s}(\dot{\mathbf{M}}^{(k)}_{\{k,l\}})\big)\\ 
		\mathbf{0}
	\end{array} \right],
\end{equation*}
and $\sigma_{n}(\dot{\mathbf{M}}^{(k)}_{\{k,l\}})=	\left\{
	\begin{array}{lc}
		0,\qquad &\text{if}\  c_{2}<0\\
		\frac{c_{1}+\sqrt{c_{2}}}{2}, &\text{if}\  c_{2}\geq0
	\end{array}
	\right.
$, with $c_{1}=\sigma_{n}(\dot{\mathbf{\Gamma}})-\epsilon$, $c_{2}=(\sigma_{n}(\dot{\mathbf{\Gamma}})+\epsilon)^{2}-4C$, and $C$ is a compromising constant.

\textbf{Update $\dot{\mathcal{T}}$}:  To optimize $\dot{\mathcal{T}}$ is equivalent to solve the subproblem:
\begin{equation}\label{equupt1}
		\begin{split}
		\dot{\mathcal{T}}=&\mathop{{\rm{arg\, min}}}\limits_{P_{\Omega}(\dot{\mathcal{T}})=P_{\Omega}(\dot{\mathcal{X}})}\ \sum_{k=2}^{N}\mathfrak{R}(\langle\dot{\mathcal{Y}}_{k},\dot{\mathcal{T}}-\dot{\mathcal{M}}^{(k)}\rangle)
		+\frac{\mu_{k}}{2}\|\dot{\mathcal{T}}-\dot{\mathcal{M}}^{(k)}\|_{F}^{2} \\
		=&\mathop{{\rm{arg\, min}}}\limits_{P_{\Omega}(\dot{\mathcal{X}})=P_{\Omega}(\dot{\mathcal{T}})}\ \sum_{k=2}^{N}\frac{\mu_{k}}{2}\|\dot{\mathcal{T}}-\dot{\mathcal{M}}^{(k)}+\frac{\dot{\mathcal{Y}}^{(k)}}{\mu_{k}}\|_{F}^{2}
	\end{split}
\end{equation}
It is easy to check that the solution of (\ref{equupt1}) is given by:
\begin{equation}\label{equupt2}
	\dot{\mathcal{T}}=P_{\Omega^{c}}\bigg(\frac{\sum_{k=2}^{N}\big(\dot{\mathcal{M}}^{(k)}-\frac{\dot{\mathcal{Y}}^{(k)}}{\mu_{k}}\big)}{N-1}\bigg)+P_{\Omega}(\dot{\mathcal{X}}),
\end{equation}
where $\Omega^{c}$ is the complement of $\Omega$.

\textbf{Update $\dot{\mathcal{Y}}^{(k)}$}: The Lagrange multiplier $\dot{\mathcal{Y}}^{(k)}$ is updated by:
\begin{equation}\label{equupy}
	\dot{\mathcal{Y}}^{(k)}=\dot{\mathcal{Y}}^{(k)}+\mu_{k}(\dot{\mathcal{T}}-\dot{\mathcal{M}}^{(k)}).
\end{equation}

To speed up convergence, each iteration we also update $\mu_{k}$ by: $\mu_{k}=\min(\mu_{max}, \rho\mu_{k})$, where $\mu_{max}$ is the
default maximum of $\mu_{k}$,  $\rho>1$ is a constant parameter.

Finally, the proposed LRQTC algorithm is summarized in Algorithm \ref{alg:qtrnnm}.
\begin{algorithm}
	\caption{Our proposed LRQTC algorithm.}
	\label{alg:qtrnnm}
	\begin{algorithmic}
	\REQUIRE The observed $N$-th order quaternion tensor $\dot{\mathcal{T}}\in \mathbb{H}^{I_{1}\times I_{2}\times \ldots \times I_{N}}$ with $\Omega$ (the index of observed elements), $\{\alpha_{k}\}_{k=2}^{N}$, $\mu_{\max}$ and $\rho$.
	\STATE \textbf{Initialize} $\{\dot{\mathcal{M}}^{(k)}\}_{k=2}^{N}$, $\{\dot{\mathcal{Y}}^{(k)}\}_{k=2}^{N}$, and $\{\mu_{k}\}_{k=2}^{N}$.
	\STATE \textbf{Repeat}
	\FOR {$k=2$ to $N$}
	\STATE Update $\dot{\mathcal{M}}^{(k)}$ via (\ref{equupm2});
	\ENDFOR
	\STATE Update $\dot{\mathcal{T}}$ via (\ref{equupt2}) (the updated one is labled by $\widetilde{\dot{\mathcal{T}}}$).
	\FOR {$k=2$ to $N$}
	\STATE Update $\dot{\mathcal{Y}}^{(k)}$ via (\ref{equupy});
	\STATE Update $\mu_{k}$ via $\mu_{k}=\min(\mu_{max}, \rho\mu_{k})$.
	\ENDFOR
	\STATE \textbf{Until} $\frac{\|\dot{\mathcal{T}}-\widetilde{\dot{\mathcal{T}}}\|_{F}}{\|\widetilde{\dot{\mathcal{T}}}\|_{F}}<10^{-5}$ or reach the preset maximum number of iterations.
	\ENSURE  \text{The recovered quaternion tensor}\  $\widetilde{\dot{\mathcal{T}}}$.
	\end{algorithmic}
\end{algorithm}

\section{Experiments and results}\label{sec:main_3_3}
In this section, we will first elaborate on how to utilize the proposed LRQTC model for color image inpainting and then present the experimental results.

\subsection{Color image inpainting}\label{sec:main_3}
Because a color image is essentially a quaternion matrix (a second-order quaternion tensor), it is necessary to increase the order of the quaternion matrix in order to effectively utilize the proposed LRQTC method. Recently, the overlapping ket augmentation (OKA) as a tensor augmentation technique was developed in \cite{zhang2022effective} for increasing the order of tensors.
OKA is an improvement upon KA \cite{bengua2017efficient} as it overcomes the visual flaws caused by reshaping and eliminates the blocking artifacts introduced by KA. Therefore, in order to increase the order of quaternion matrices used for representing color images, we apply OKA to quaternion matrices. Due to the similarity in the process of applying OKA to tensors \cite{zhang2022effective} and quaternion matrices, we will not reiterate it here. Finally, we summarize the proposed entire process of color image inpainting in Figure \ref{imageinpainting_fig}.
\begin{figure*}[htbp]
	\centering
	\includegraphics[width=13cm,height=7cm]{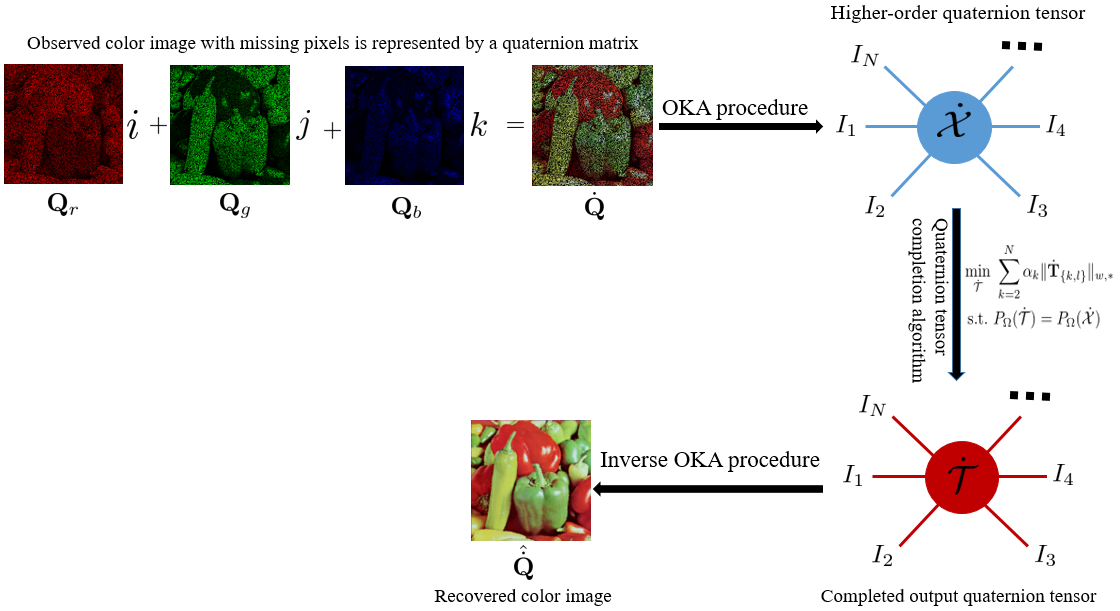} 
	\caption{The entire process of color image inpainting.}
	\label{imageinpainting_fig}
\end{figure*}

\subsection{Experimental results}
\label{sec:main_4}
To validate the effectiveness of our color image inpainting method, we conducted extensive experiments using a diverse range of images, including natural color images, color medical images, and color face images. We compare our proposed method with several classic and state-of-the-art quaternion matrix and tensor completion methods, including t-SVD \cite{zhang2016exact}, TMac-TT \cite{bengua2017efficient}, TRLRF \cite{yuan2019tensor}, TRNNM \cite{huang2020provable}, LRQA-2 \cite{chen2019low}, LRQMC \cite{miao2021color}, and TQLNA \cite{yang2022quaternion}. In order to assess the performance of the proposed method, we considered not only visual quality but also utilized two commonly used quantitative quality metrics: Peak Signal-to-Noise Ratio (PSNR) and Structural Similarity Index (SSIM) \cite{wang2004image}. All the experiments are run in MATLAB $2014b$ under Windows $10$ on a personal computer with a $1.60GHz$ CPU and $8GB$ memory. 

\textbf{Natural color image inpainting}:
Five natural color images (shown in the first row of Figure \ref{test_image}) with a spatial resolution of $256 \times 256$ are utilized for the evaluation.  For our proposed method, the natural color images are transformed into ninth-order quaternion tensors of size $4\times4\times4\times4\times4\times4\times4\times4\times4$ using OKA. We set $\alpha_{k}=\frac{\omega_{k}}{\sum_{k=2}^{N}\omega_{k}}$ with $\omega_{k}=\min(\prod_{n=1}^{k-1}I_{n},\prod_{n=k}^{N}I_{n})$ for $k=2,3,\ldots,N$, $\mu_{\max}=10^6$, and $\rho=1.03$. We initialize $\dot{\mathcal{M}}^{(k)}=\dot{\mathcal{T}}$, $\dot{\mathcal{Y}}^{(k)}=\mathbf{0}$ for $k=2,3,\ldots,N$, and $\mu=\{0.5,0.5,0.001,10^{-4.1},10^{-4.1},0.001,0.5,0.5\}$. Furthermore, all the compared methods were implemented using their source codes, and the parameter configurations were set according to the recommendations provided in the original papers, with adjustments made to optimize performance as closely as possible.

\textbf{Color medical image inpainting}: Five color medical images (shown in the second
row of Figure \ref{test_image}) with a spatial resolution of $256 \times 256$ are utilized for the evaluation. The experimental settings are the same as those for natural color image inpainting.

\textbf{Color face image inpainting}: Five color face images (shown in the third
row of Figure \ref{test_image}) with a spatial resolution of $120 \times 165$ are utilized for the evaluation. For our proposed method, the color face images are transformed into eighth-order quaternion tensors of size $4\times4\times4\times4\times4\times4\times5\times4$ using OKA. The other experimental settings are the same as those for natural color image inpainting.

\begin{figure*}[htbp]
	\centering
	\includegraphics[width=10cm,height=6cm]{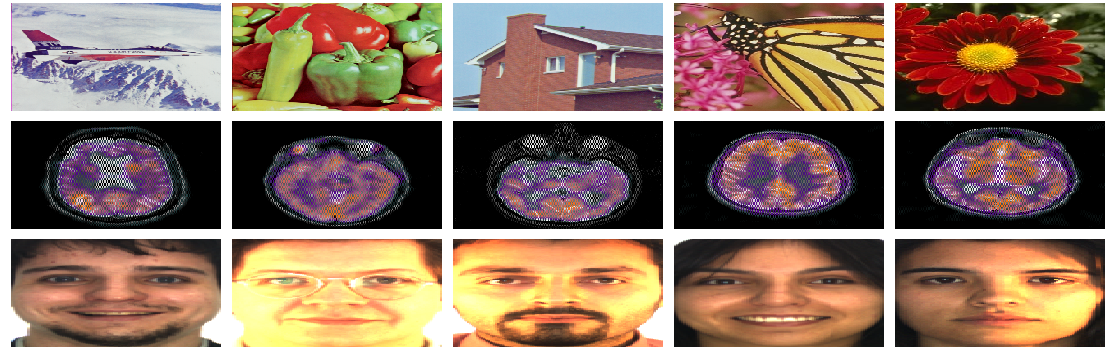} 
	\caption{ The tested natural color images (the first row), color medical images (the second row), and color face images (the third row).}
	\label{test_image}
\end{figure*}

\begin{table*}[htbp]
	\caption{Average PSNR and SSIM values (PSNR, SSIM) on the five {\rm{\textbf{natural color images}}} with five levels of sampling rates (SRs) (\textbf{bold} fonts denote the best performance).}
	\centering
	\resizebox{13cm}{1.8cm}{
		\begin{tabular}{|c|ccccc|}		
			\hline
			\diagbox{Methods}{SRs}&{\rm{SR}}=10\%&{\rm{SR}}=20\%&{\rm{SR}}=30\%&{\rm{SR}}=40\%& {\rm{SR}}=50\% \\ 
			\hline	
			t-SVD \cite{zhang2016exact} &17.266, 0.638&20.206, 0.766&22.639, 0.840&24.854, 0.891&27.176, 0.929\\
			\hline
			TMac-TT \cite{bengua2017efficient}
			&20.480, 0.778&22.415, 0.847&24.651, 0.903&26.229, 0.931&28.277, 0.956\\
			\hline
			TRLRF \cite{yuan2019tensor}&17.374, 0.633&20.324, 0.761&23.217, 0.851&25.609, 0.903&27.997, 0.940\\
			\hline
			TRNNM \cite{huang2020provable} &19.683, 0.801&22.743, 0.880&24.938, 0.920&26.841, 0.945&28.778, 0.963\\
			\hline
			LRQA-2 \cite{chen2019low} &18.063, 0.663&20.950, 0.781&23.241, 0.848&25.337, 0.894&27.545, 0.928\\
			\hline
			LRQMC \cite{miao2021color} &17.738, 0.677&20.838, 0.797&23.367, 0.863&25.589, 0.908&27.976, 0.942\\
			\hline
			TQLNA \cite{yang2022quaternion} &17.819, 0.658&21.124, 0.788&23.632, 0.859&25.870, 0.905&28.252, 0.939\\
			\hline
			\textbf{Ours} &\textbf{23.276}, \textbf{0.881}&\textbf{25.933}, \textbf{0.929}&\textbf{27.948}, \textbf{0.953}&\textbf{29.679}, \textbf{0.968}&\textbf{31.310}, \textbf{0.977}\\
			\hline
	\end{tabular}}
	\label{natural_color}
\end{table*}

\begin{figure*}[htbp]
	\centering
	\includegraphics[width=13cm,height=15cm]{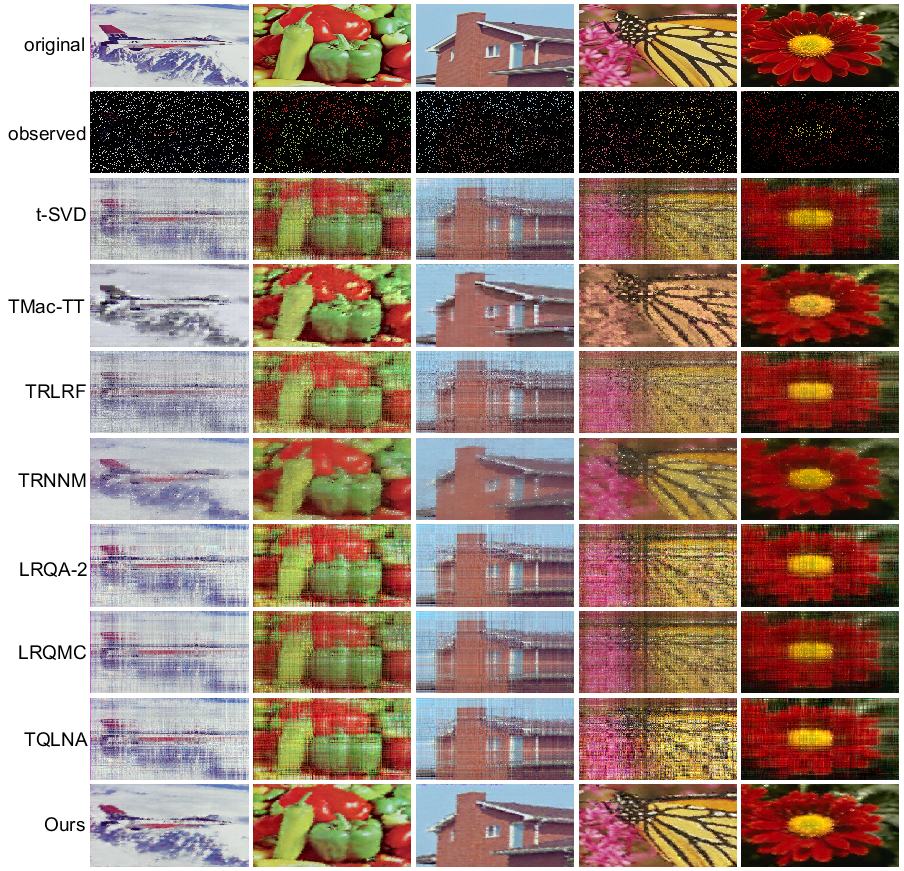} 
	\caption{Recovered natural color images for random missing with $SR = 10\%$.  From the top row to the bottom row: the original color images, the observed color images, the results recovered by t-SVD, TMac-TT, TRLRF, TRNNM, LRQA-2, LRQMC, TQLNA, and our method, respectively. \textbf{The figure
		is viewed better in zoomed PDF}.}
	\label{natural_color_01}
\end{figure*}

\begin{table*}[htbp]
	\caption{Average PSNR and SSIM values (PSNR, SSIM) on the five {\rm{\textbf{color medical images}}} with five levels of sampling rates (SRs) (\textbf{bold} fonts denote the best performance).}
	\centering
	\resizebox{13cm}{1.8cm}{
		\begin{tabular}{|c|ccccc|}		
			\hline
			\diagbox{Methods}{SRs}&{\rm{SR}}=10\%&{\rm{SR}}=20\%&{\rm{SR}}=30\%&{\rm{SR}}=40\%& {\rm{SR}}=50\% \\ 
			\hline	
			t-SVD \cite{zhang2016exact} &17.568, 0.575&19.816, 0.688&21.684, 0.761&23.383, 0.814&25.013, 0.856\\
			\hline
			TMac-TT \cite{bengua2017efficient}
			&20.139, 0.714&21.980, 0.806&23.937, 0.869&25.668, 0.910&27.192, 0.934\\
			\hline
			TRLRF \cite{yuan2019tensor}&17.142, 0.233&19.720, 0.420&21.686, 0.560&23.221, 0.638&24.771, 0.712\\
			\hline
			TRNNM \cite{huang2020provable} &17.982, 0.639&21.340, 0.797&23.884, 0.875&26.142, 0.922&28.610, 0.957\\
			\hline
			LRQA-2 \cite{chen2019low} &18.099, 0.590&20.366, 0.695&22.063, 0.760&23.662, 0.811&25.271, 0.851\\
			\hline
			LRQMC \cite{miao2021color} &17.687, 0.591&20.025, 0.702&21.932, 0.775&23.586, 0.826&25.449, 0.871\\
			\hline
			TQLNA \cite{yang2022quaternion} &17.899, 0.597&20.589, 0.714&22.350, 0.776&23.937, 0.824&25.527, 0.863\\
			\hline
			\textbf{Ours} &\textbf{22.339}, \textbf{0.836}&\textbf{24.553}, \textbf{0.894}&\textbf{26.111}, \textbf{0.923}&\textbf{27.608}, \textbf{0.944}&\textbf{29.160}, \textbf{0.960}\\
			\hline
	\end{tabular}}
	\label{color_medical}
\end{table*}
\begin{figure*}[htbp]
	\centering
	\includegraphics[width=13cm,height=15cm]{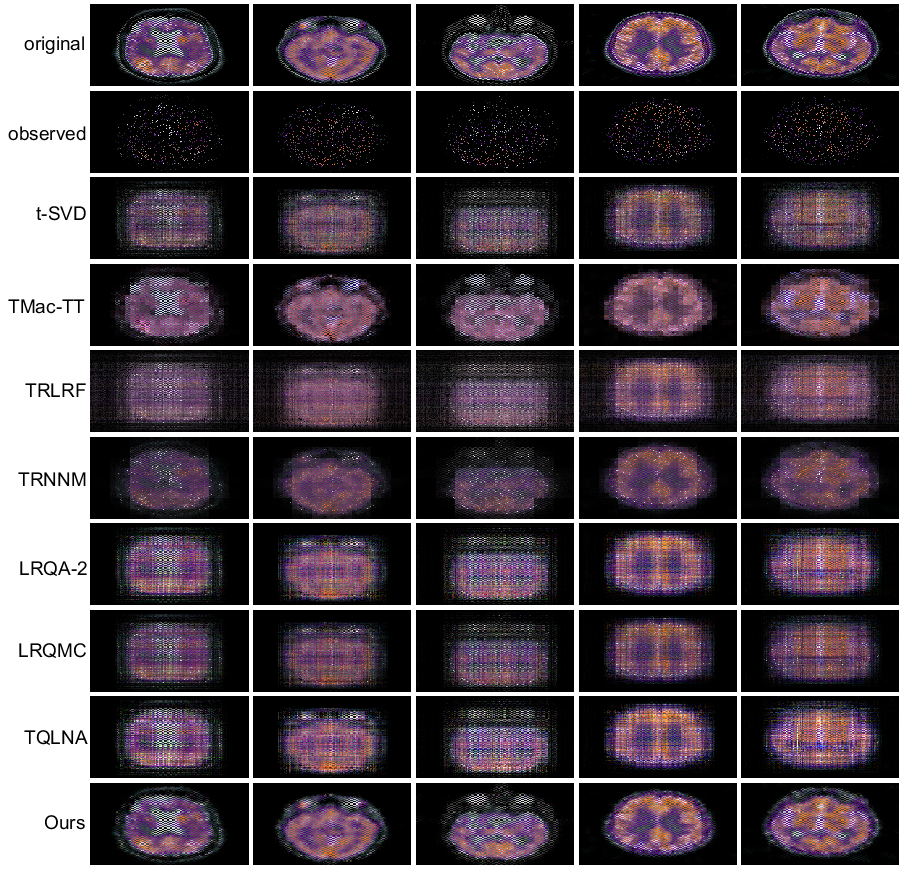} 
	\caption{Recovered color medical images for random missing with $SR = 10\%$. From the top row to the bottom row: the original color images, the observed color images, the results recovered by t-SVD, TMac-TT, TRLRF, TRNNM, LRQA-2, LRQMC, TQLNA, and our method, respectively. \textbf{The figure
			is viewed better in zoomed PDF}.}
	\label{color_medical_01}
\end{figure*}

\begin{table*}[htbp]
	\caption{Average PSNR and SSIM values (PSNR, SSIM) on the five {\rm{\textbf{color face images}}} with five levels of sampling rates (SRs) (\textbf{bold} fonts denote the best performance).}
	\centering
	\resizebox{13cm}{1.8cm}{
		\begin{tabular}{|c|ccccc|}		
			\hline
			\diagbox{Methods}{SRs}&{\rm{SR}}=10\%&{\rm{SR}}=20\%&{\rm{SR}}=30\%&{\rm{SR}}=40\%& {\rm{SR}}=50\% \\ 
			\hline	
			t-SVD \cite{zhang2016exact} &18.416, 0.789&22.761, 0.885&25.984, 0.932&28.515, 0.955&31.317, 0.972\\
			\hline
			TMac-TT \cite{bengua2017efficient}
			&22.862, 0.890&28.633, 0.964&31.093, 0.978&32.952, 0.984&34.534, 0.988\\
			\hline
			TRLRF \cite{yuan2019tensor}&18.424, 0.784&22.422, 0.867&25.248, 0.919&27.554, 0.945&29.689, 0.962\\
			\hline
			TRNNM \cite{huang2020provable} &22.452, 0.916&26.878, 0.958&29.508, 0.974&31.729, 0.982&33.940, 0.989\\
			\hline
			LRQA-2 \cite{chen2019low} &20.021, 0.817&24.148, 0.901&26.908, 0.940&29.323, 0.960&31.674, 0.973\\
			\hline
			LRQMC \cite{miao2021color} &19.172, 0.819&23.216, 0.897&26.516, 0.938&28.423, 0.956&31.132, 0.971\\
			\hline
			TQLNA \cite{yang2022quaternion} &19.585, 0.809&24.160, 0.900&27.347, 0.943&29.902, 0.965&32.901, 0.979\\
			\hline
			\textbf{Ours} &\textbf{26.708}, \textbf{0.953}&\textbf{30.050}, \textbf{0.976}&\textbf{32.332}, \textbf{0.984}&\textbf{34.288}, \textbf{0.989}&\textbf{35.990}, \textbf{0.992}\\
			\hline
	\end{tabular}}
	\label{color_face}
\end{table*}
\begin{figure*}[htbp]
	\centering
	\includegraphics[width=13cm,height=15cm]{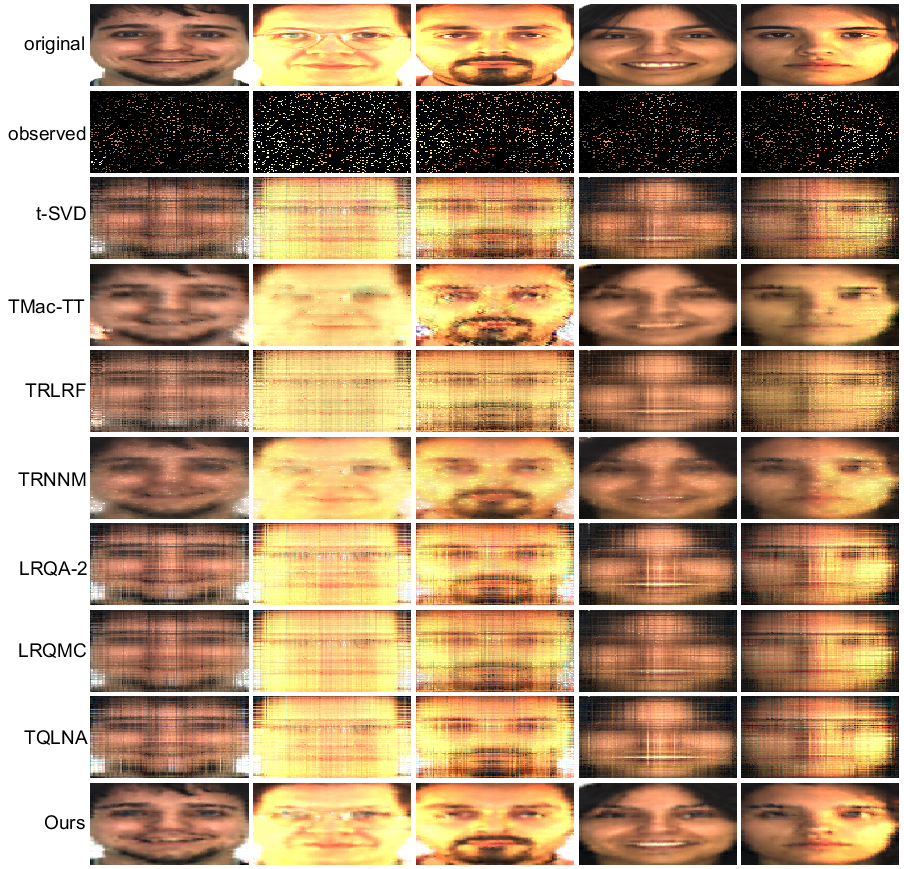} 
	\caption{Recovered color face images for random missing with $SR = 10\%$. From the top row to the bottom row: the original color images, the observed color images, the results recovered by t-SVD, TMac-TT, TRLRF, TRNNM, LRQA-2, LRQMC, TQLNA, and our method, respectively. \textbf{The figure
			is viewed better in zoomed PDF}.}
	\label{color_face_01}
\end{figure*}

\begin{figure*}[htbp]
	\centering
	\includegraphics[width=13cm,height=15cm]{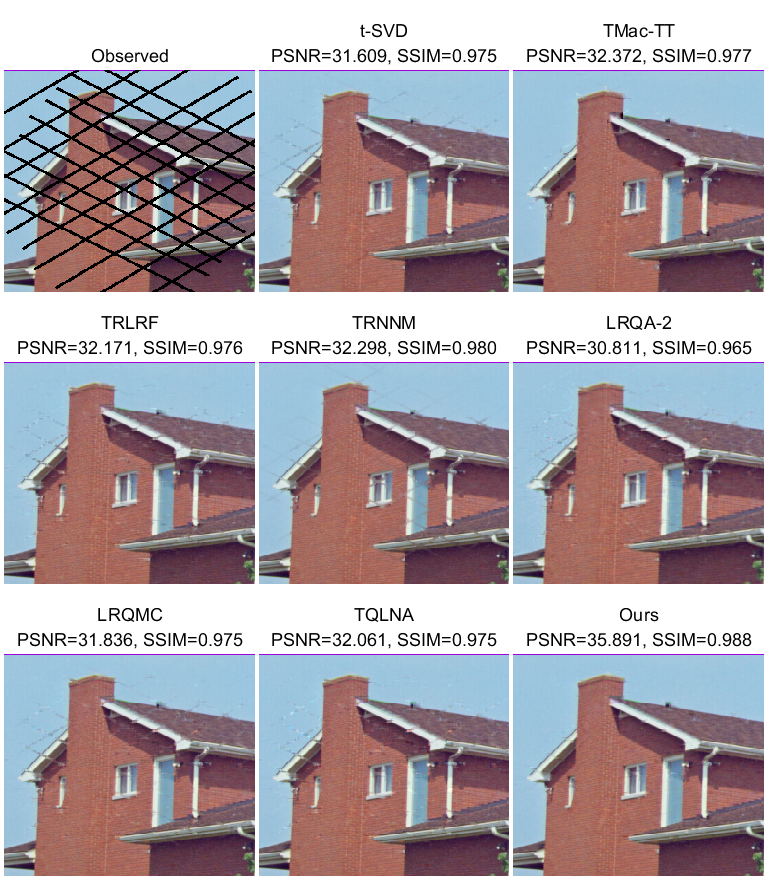} 
	\caption{Recovery of structurally missing natural color image: a comparison of various methods in terms of visual and quantitative metrics. \textbf{The figure
			is viewed better in zoomed PDF}.}
	\label{stru1}
\end{figure*}

\begin{figure*}[htbp]
	\centering
	\includegraphics[width=13cm,height=15cm]{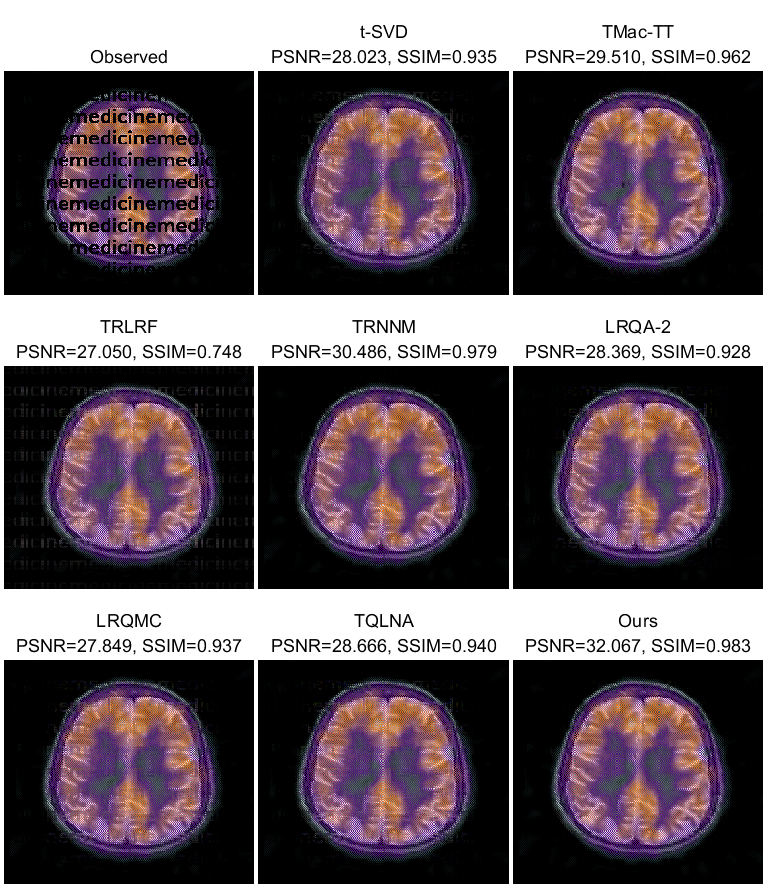} 
	\caption{Recovery of structurally missing color medical image: a comparison of various methods in terms of visual and quantitative metrics. \textbf{The figure
			is viewed better in zoomed PDF}.}
	\label{stru2}
\end{figure*}

\begin{figure*}[htbp]
	\centering
	\includegraphics[width=13cm,height=15cm]{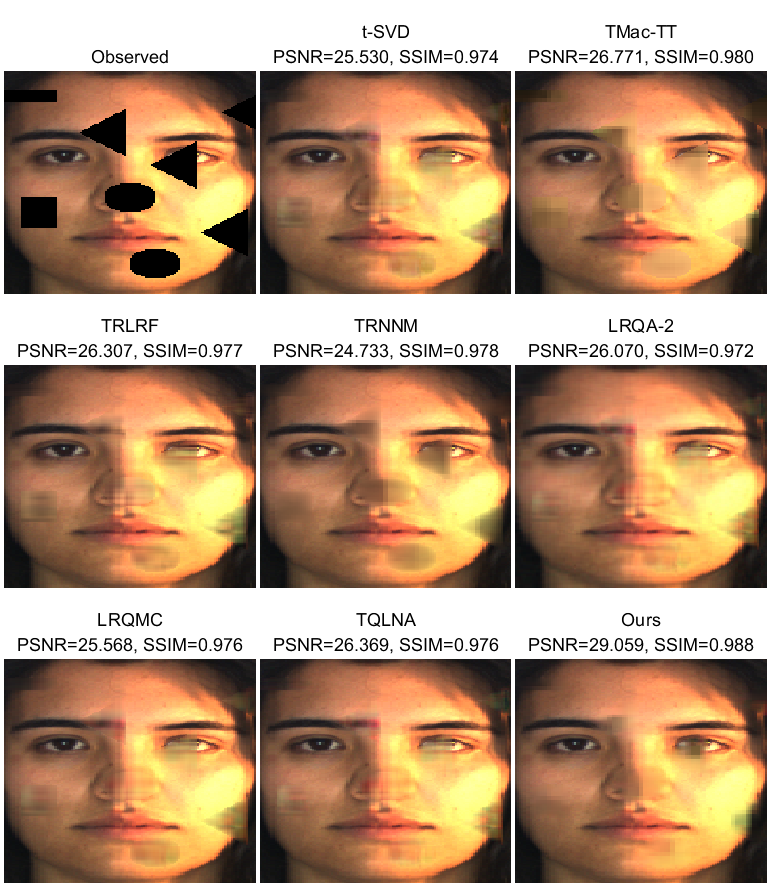} 
	\caption{Recovery of structurally missing color face image: a comparison of various methods in terms of visual and quantitative metrics. \textbf{The figure
			is viewed better in zoomed PDF}.}
	\label{stru3}
\end{figure*}

For random missing, we set five levels of sampling rates (SRs) which are ${\rm{SR}}= 10\%$, ${\rm{SR}} = 20\%$, ${\rm{SR}} = 30\%$, ${\rm{SR}} = 40\%$, and ${\rm{SR}} = 50\%$. The average PSNR and SSIM values are reported in Table \ref{natural_color},  Table \ref{color_medical}, and Table \ref{color_face} for the five natural color images,  color medical images, and color face images at five different levels of SRs. Figures \ref{natural_color_01}-\ref{color_face_01} visually demonstrate the recovered results obtained by various methods for natural color images, color medical images, and color face images at ${\rm{SR}}=10\%$. Furthermore, we also validated the performance of these methods in recovering color images with structural missing. We conducted experiments on randomly selected one color image from each of the three categories, and the visual and quantitative results are shown in Figures \ref{stru1}-\ref{stru3}. From these extensive experiments, it is evident that our proposed method exhibits significant advantages both visually and in terms of quantitative metrics when compared to both quaternion-based methods and TR-based methods. This aligns with our expectation of complementary strengths between quaternions and TR. Additionally, we observed that these methods perform poorly when recovering color images with large areas of complete loss, as illustrated in Figure \ref{stru3}. While our proposed method is relatively acceptable visually, the recovery of edges remains less than ideal. Hence, for inpainting tasks involving color images with large areas of complete loss, there is a requirement for specific model improvements in forthcoming research.

\section{Conclusions}
\label{sec:main_5}

In this paper, we have defined the QTLR decomposition of quaternion tensors and introduced the relevant theory. The QTLR decomposition combines the advantages of both quaternions and TR decomposition, providing a new theoretical foundation for the field of color image processing. This constitutes the core research focus of this paper. Furthermore, as an example of the application of the QTLR decomposition, we have proposed a method for color image inpainting. Specifically, we define the circular unfolding of quaternion tensors, establish a correlation between the rank of circular unfolding quaternion matrices and QTLR-rank, and leverage this connection to propose an LRQTC model for color image inpainting. The experiments provide evidence that the LRQTC method we introduce showcases exceptional performance across a spectrum of color image inpainting tasks. Regardless of whether compared to existing quaternion matrix-based methods or TR-based methods, our approach exhibits significant advantages both visually and in terms of quantitative metrics.
This underscores the substantial potential inherent in the fusion of tensor TR with quaternions. 

Nevertheless, owing to the non-commutative nature of quaternion multiplication, QTLR has yet to attain the same degree of theoretical robustness as TR. For instance, concerning the learning algorithm for the QTLR format, we have put forth solely a QTLR-QSVD approach. Ensuring the validity of Theorem \ref{theor1}, in which (\ref{rankthe}) holds, entails the imposition of constraints such as $l=N-k+1$, among others. These aspects underscore the need for additional refinement and enhancement in our forthcoming research endeavors.

\bibliographystyle{siamplain}
\bibliography{references}
\end{document}

%% file: ex_shared.tex
\usepackage{lipsum}
\usepackage{amsfonts}
\usepackage{graphicx}
\usepackage{epstopdf}
\usepackage{algorithmic}
\ifpdf
  \DeclareGraphicsExtensions{.eps,.pdf,.png,.jpg}
\else
  \DeclareGraphicsExtensions{.eps}
\fi

\newsiamremark{remark}{Remark}
\newsiamremark{hypothesis}{Hypothesis}
\crefname{hypothesis}{Hypothesis}{Hypotheses}
\newsiamthm{claim}{Claim}

\headers{}{Jifei Miao, Kit Ian Kou, Hongmin Cai, and Lizhi Liu}

\title{Quaternion tensor left ring decomposition and application for color image inpainting \thanks{Submitted to the editors. 
\funding{The work of the second author was supported by the University of Macau (MYRG2022-00108-FST, MYRG-CRG2022-00010-ICMS), The Science and Technology Development Fund, Macau S.A.R (0036/2021/AGJ). The work of the third author was supported in part by the National Key Research and Development Program of China (2022YFE0112200), the National Natural Science Foundation of China (U21A20520, 62325204), Science and Technology Project of Guangdong Province (2022A0505050014), the Key-Area Research and Development Program of Guangzhou City (202206030009). The work of the fourth author was supported by the Science and Technology Planning Project of Guangzhou City, China (201907010043).}}}

\author{Jifei Miao\thanks{The School of Mathematics and Statistics, Yunnan University, Kunming, Yunnan, 650091, China
  (\email{jifmiao@163.com}).}
\and Kit Ian Kou\thanks{The Department of Mathematics, Faculty of
	Science and Technology, University of Macau, Macau 999078, China
  (\email{kikou@umac.mo}).}
\and Hongmin Cai\thanks{The School of Computer Science and Engineering, South China University of Technology, Guangzhou, Guangdong, 510641, China
	(\email{hmcai@scut.edu.cn}).}
\and Lizhi Liu\thanks{The Department of Radiology, Sun Yat-sen University Cancer Center,State Key Laboratory of Oncology in South China, Collaborative Innovation Center for Cancer Medicine, Guangdong Key Laboratory of Nasopharyngeal Carcinoma Diagnosis and Therapy, 651 Dongfeng Road East, Guangzhou, Guangdong, 510060, China
	(\email{liulizh@sysucc.org.cn}).}
}

\usepackage{amsopn}
